\RequirePackage{fix-cm}

\documentclass{article}       
\usepackage{graphicx}
\usepackage{url}

\usepackage{amsmath}
\usepackage{amsthm}
\usepackage{algorithm2e}
\usepackage{enumitem}

\newtheorem{theorem}{Theorem}
\newtheorem{lemma}{Lemma}

\newcommand{\dataset}{{\cal D}}
\newcommand{\calX}{\mathcal{X}}
\newcommand{\calA}{\mathcal{A}}

\DeclareMathOperator*{\argmin}{arg\,min}

\begin{document}

\title{Boosting algorithms for uplift modeling\thanks{This work was supported by Research Grant no.~N~N516~414938 of the Polish Ministry of Science and Higher Education (Ministerstwo Nauki i Szkolnictwa Wy\.zszego) from research funds for the period 2010--2014. M.S. was also supported by the European Union from resources of the European Social Fund: Project POKL `Information technologies: Research and their interdisciplinary applications', Agreement UDA-POKL.04.01.01-00-051/10-00.}}


\author{Micha{\l} So{\l}tys\thanks{Institute of Computer Science, Polish Academy of Sciences, Warsaw, Poland, E-mail: {\tt s.jaroszewicz@ipipan.waw.pl}} \and Szymon Jaroszewicz\footnotemark[2]}


\date{}

\maketitle

\begin{abstract}
Uplift modeling is an area of machine learning which aims at
predicting the causal effect of some action on a given individual. The
action may be a medical procedure, marketing campaign, or any other
circumstance controlled by the experimenter.  Building an uplift model
requires two training sets: the treatment group, where individuals
have been subject to the action, and the control group, where no
action has been performed. An uplift model allows then to assess the
gain resulting from taking the action on a given individual, such as
the increase in probability of patient recovery or of a product being
purchased.  This paper describes an adaptation of the well-known
boosting techniques to the uplift modeling case. We formulate three
desirable properties which an uplift boosting algorithm should
have. Since all three properties cannot be satisfied
simultaneously, we propose three uplift boosting algorithms, each
satisfying two of them. Experiments demonstrate the usefulness of the
proposed methods, which often dramatically improve performance of the
base models and are thus new and powerful tools for uplift modeling.
\end{abstract}

\section{Introduction}\label{sec:intro}

Machine learning is primarily concerned with the problem of
classification, where the task is to predict, based on a number of
predictor attributes, the class to which an instance belongs.
Unfortunately, classification is not well suited to many problems in
marketing or medicine to which it is frequently applied.  Consider a
direct marketing campaign where potential customers receive a mailing
offer.  A classifier is typically built based on a small pilot
campaign and used to select the customers who should be targeted.  As
a result, the customers most likely to buy {\em after} the campaign
will be selected as targets.  Unfortunately this is not what a
marketer wants.  Some of the customers would have bought regardless of
the campaign, targeting them resulted in unnecessary costs.  Other
customers were actually going to make a purchase but were annoyed by
the campaign.  While, at first sight, such a case may seem unlikely, it
is a well known phenomenon in the marketing
literature~\cite{Hansotia2002,RadcliffeTechRep}; the result is a loss
of a sale or even churn.

We should therefore select customers who will buy {\em because} of the
campaign, that is, those who are likely to buy if targeted, but
unlikely to buy otherwise.  Similar problems arise in medicine where
some patients may recover without treatment and some may be hurt by
treatment's side effects more than by the disease itself.

{\em Uplift modeling} provides a solution to this problem. The
approach uses two separate training sets: {\em treatment} and {\em
  control}.  Individuals in the treatment group have been subjected to
the action, those in the control group have not.  Instead of modeling
class probabilities, uplift modeling attempts to model the {\em
  difference} between conditional class probabilities in the treatment
and control groups.  This way, the causal influence of the
action can be modeled, and the method is able to predict the true gain
(with respect to taking no action) from targeting a given individual.

This paper presents an adaptation of boosting to the uplift modeling
case.  Boosting often dramatically improves performance of
classification models, and in this paper we demonstrate that it can
bring similar benefits to uplift modeling.  We begin by stating three
desirable properties of an uplift boosting algorithm.  Since all three
cannot be satisfied at the same time, we propose three uplift boosting
algorithms, each satisfying two of them.  Experimental verification
proves that the benefits of boosting extend to the case of uplift
modeling and shows relative merits of the three algorithms.

We conclude by mentioning a problem which is the biggest challenge of
uplift modeling as opposed to standard classification.  The problem
has been known in statistical literature~\cite{Holland86} as the

\begin{quote}
  {\bf Fundamental Problem of Causal Inference.}
  For every individual, only one of the outcomes is observed, after the
  individual has been subject to an action (treated) or when the
  individual has not been subject to the action (was a control case),
  {\em never} both.
\end{quote}

As a result, we never know whether the action performed on a given
individual was truly beneficial.  This is different from
classification, where the true class of each individual in the
training set is known.

In the remaining part of this section we describe the notation, give
an overview of the related work and review some of the properties of
classification boosting.

\subsection{Notation and assumptions}

We will now introduce the notation used throughout the paper. We use
the superscript ${}^T$ for quantities related to the treatment group
and the superscript ${}^C$ for quantities related to the control
group. For example, the treatment training dataset will be denoted
with $\dataset^T$ and the control training dataset with $\dataset^C$.
Both datasets together constitute the whole training dataset,
$\dataset = \dataset^T \cup \dataset^C$.

Each data record $(x,y)$ consists of a vector of features $x\in\calX$
and a class $y\in\{0,1\}$ with $1$ assumed to be the successful
outcome, for example patient recovery or a positive response to a
marketing campaign.  Let $N^T$ and $N^C$ denote the number of records
in the treatment and control datasets.

An uplift model is a function $h:\calX\rightarrow\{0,1\}$.  The value
$h(x)=1$ means the action is deemed beneficial for $x$ by the model,
$h(x)=0$ means that its impact is considered neutral or negative.  By
`positive outcome' we mean that the probability of success for a given
individual $x$ is higher if the action is performed on her than if the
action is not taken.

We will denote general probabilities related to the treatment and
control groups with $P^T$ and $P^C$, respectively.  For example,
$P^T(y=1,h=1)$ stands for probability that a randomly selected case in
the treatment set has a positive outcome and taking the action on it
is predicted to be beneficial by an uplift model $h$.  We can now
state more formally when an individual $x$ should be subject to an
action, namely, when $P^T(y=1|x)-P^C(y=1|x)>0$.

\subsection{Related work}\label{sec:rel-work}

Despite its practical appeal, uplift modeling has seen relatively
little attention in the literature.  A trivial approach is to build
two probabilistic classifiers, one on the treatment set, the other on
control, and subtract their predicted probabilities.  This approach
can, however, suffer from a serious drawback: both models may focus on
predicting class probabilities themselves, instead of focusing on the,
usually much smaller, differences between the two groups.  An
illustrative example can be found in~\cite{RadcliffeTechRep}.

Several algorithms have thus been proposed which directly model the
difference between class probabilities in the treatment and control
groups. Many of them are based on modified decision trees.  For
example,~\cite{RadcliffeTechRep} describe an uplift tree learning
algorithm which selects splits based on a statistical test of
differences between treatment and control class probabilities.
In~\cite{my:uplift-trees,my:uplift-trees-KAIS} uplift decision trees
based on information theoretical split criteria have been proposed.
Since we use those trees as base models, we will briefly discuss the
splitting criterion they use.

The splitting criterion for uplift decision trees compares some
measure of divergence between treatment and control class
probabilities before and after the split.  The difference between two
probability distributions $P=(p_1,\ldots,p_n)$ and
$Q=(q_1,\ldots,q_n)$ is typically measured using the Kullback-Leibler
divergence~\cite{csiszar}, however the authors found that
E-divergence $$E(P:Q) = \sum_i (p_i - q_i)^2,$$ which is simply the
squared Euclidean distance, performed better in the experiments.  The
test selection criterion (the E-divergence gain) then becomes
\[
E_{gain}(A) = \sum_{a\in\calA}P(a)E\left(P^T(Y|a):P^C(Y|a)\right) - E\left(P^T(Y):P^C(Y)\right),
\]
where $\calA$ is the set of possible outcomes of the test $A$ and
$P(a)$ is a weighted average of probabilities of outcome $a$ in the
treatment and control training sets.  The value is additionally
divided by a penalizing factor which discourages tests with a large
number of outcomes or tests which lead to very different splits in the
treatment and control training sets.  Details can be found
in~\cite{my:uplift-trees,my:uplift-trees-KAIS}.

Some work has also been published on using ensemble methods for uplift
modeling, although, to the best of our knowledge, none of them on
boosting.  Bagging of uplift models has been mentioned
in~\cite{RadcliffeTechRep}.  Uplift Random Forests have been proposed
by~\cite{GuelmanRandForest}; an extension, called causal conditional
inference trees was proposed by the same authors in~\cite{Guelman2}.
A thorough experimental and theoretical analysis of bagging and random
forests in uplift modeling can be found in~\cite{my:uplEnsembles}
where it is argued that ensemble methods are especially well suited to
this task and that bagging performs surprisingly well.

Other uplift techniques have also been proposed.  Regression based
approaches can be found in~\cite{Lo2002} or, in a medical context,
in~\cite{Robins2004,Goetghebeur2003}.
In~\cite{my:uplift-clinical-ml} a class variable transformation has
been proposed, which allows for applying arbitrary classifiers to
uplift modeling, however the performance of the method was not
satisfactory as it was often outperformed by the double model
approach.  \cite{my:upl-surv} propose a method for converting survival
data such that uplift modeling can, under certain assumptions, be
directly applied to it; we will use this method to prepare
experimental datasets in Section~\ref{sec:exper}.

Some variations on the uplift modeling theme have also been explored.
\cite{pechyony13} proposed an approach to online advertising which
combines uplift modeling with maximizing the response rate in the
treatment group to increase advertiser's benefits.  We do not
address such problems in this paper.

\subsection{Properties of boosting in the classification case}

While many boosting algorithms are available, in this paper by
`boosting' we mean the discrete AdaBoost algorithm~\cite{freund:97}.

We will now briefly summarize two of the properties of boosting in the
classification case.  The two properties will be important for
adapting boosting to the uplift modeling case.  Full details can be
found for example
in~\cite{freund:97,Schapire:1990,schapire1999improved}.

\begin{description}
\item[Forgetting the last member.] The first important property is
  `forgetting' the last model added to the ensemble.  After a new
  member is added, record weights are updated such that its
  classification error is exactly $1/2$~\cite{schapire1999improved}.
  This makes it likely for the next member to be very different from
  the previous one, leading to a diverse ensemble.
\item[Fastest decrease of ensemble's training error.] It can be shown~\cite{freund:97},
  that training set error of a boosted ensemble is bounded from above
  by
\[
\epsilon\leq 2^M\prod_{m=1}^M\sqrt{\epsilon_m(1-\epsilon_m)},
\]
where $\epsilon_m$ is the weighted error of the $m$-th model.  The
error thus decreases exponentially as long as
$\epsilon_m<1/2$. Moreover, the updates to record weights as well as
the weights of ensemble members are chosen such that this upper bound
is minimized.
\end{description}

\section{Boosting in the context of uplift modeling}\label{sec:uplift_boosting}

In this section we will present a general uplift boosting algorithm,
define an uplift analogue of classification error, and state three
properties which uplift boosting algorithms should have.

\subsection{A general uplift boosting algorithm}\label{sec:general-uplift-boosting}

In the $m$-th iteration of the boosting algorithm the $i$-th treatment
group training record is assumed to have a weight $w_{m,i}^T$ assigned
to it.  Likewise a weight $w_{m,i}^C$ is assigned to the $i$-th
control training case.  Further, denote by
\begin{equation}\label{eq:def-ptpc}
p_m^T = \frac{\sum_{i=1}^{N^T} w_{m,i}^T}{\sum_{i=1}^{N^T} w_{m,i}^T + \sum_{i=1}^{N^C} w_{m,i}^C},
\;\;\;\;\; p_m^C = \frac{\sum_{i=1}^{N^C} w_{m,i}^C}{\sum_{i=1}^{N^T} w_{m,i}^T + \sum_{i=1}^{N^C} w_{m,i}^C}
\end{equation}
the relative sizes of treatment and control datasets at iteration $m$.
Notice that $p_m^T + p_m^C = 1$ for every $m$.

Algorithm~\ref{alg:upl-boost} presents a general boosting algorithm
for uplift modeling.
Overall it is similar to boosting for classification but we allow for
different weights of records in the treatment and control groups.  As
a result, model weights $\beta_m$ need not be identical to weight
rescaling factors $\beta_m^T$, $\beta^C_m$.

\begin{algorithm}
\begin{itemize}[noitemsep,nolistsep]
\item[\textbf{Input:}] set of treatment training records, $\dataset^T=\{ \left(x_1^T,y_1^T\right), \dots, \left(x_{N^T}^T,y_{N^T}^T\right) \}$,
\item[] set of control training records, $\dataset^C = \{ \left(x_1^C,y_1^C\right), \dots, \left(x_{N^C}^C,y_{N^C}^C\right) \}$,
\item[] base uplift algorithm to be boosted,
\item[] integer $M$ specifying the number of iterations
\end{itemize}
  \begin{enumerate}
    \item Initialize weights $w_{1,i}^T,w_{1,i}^C$\label{step:init}
    \item For $m\leftarrow 1,\dots,M$
      \begin{enumerate}
        \item $w^T_{m,i}\leftarrow\frac{w^T_{m,i}}{\sum_j w^T_{m,j} + \sum_j w^C_{m,j}}$; $w^C_{m,i}\leftarrow\frac{w^C_{m,i}}{\sum_j w^T_{m,j} + \sum_j w^C_{m,j}}$\label{step:norm}
        \item Build a base model $h_m$ on $\dataset$ with $w_{m,i}^T,w_{m,i}^C$\label{step:build-mod}
        \item Compute the treatment and control errors $\epsilon_m^T$, $\epsilon_m^C$
        \item Compute $\beta_m^T(\epsilon_m^T, \epsilon_m^C)$, $\beta_m^C(\epsilon_m^T, \epsilon_m^C)$
        \item If $\beta_m^T=\beta_m^C=1$ or $\epsilon_m^T \notin (0, \frac{1}{2})$ or $\epsilon_m^C \notin (0, \frac{1}{2})$:\label{step:restart-weights}
          \begin{enumerate}
          \item choose random weights $w_{m,i}^T$, $w_{m,i}^C$
          \item continue with next boosting iteration 
          \end{enumerate}
        \item $w^T_{m+1,i}\leftarrow w_{m,i}^T\cdot(\beta_m^T)^{1[h_m(x_i^T)= y^T_i]}$\label{step:upd-t}
        \item $w^C_{m+1,i}\leftarrow w_{m,i}^C\cdot(\beta_m^C)^{1[h_m(x_i^C)= 1-y^C_i]}$\label{step:upd-c}
        \item $\beta_m\leftarrow\min\{\beta_m^T, \beta_m^C\}$\label{step:model-w}
        \item Add $h_m$ with coefficient $\beta_m$ to the ensemble
      \end{enumerate}
  \end{enumerate}
\textbf{Output:} The final hypothesis
\begin{equation}\label{eq:final_hypothesis}
h_f(x) =\left\{
\begin{array}{ll}
1 & if \; \sum_{m=1}^M \left( \log\frac{1}{\beta_m} \right) h_m(x) \geq \frac{1}{2}\sum_{m=1}^M \log \frac{1}{\beta_m}, \\
0 & otherwise.
\end{array}
\right.
\end{equation}
\caption{A general uplift boosting algorithm.}\label{alg:upl-boost}
\end{algorithm}

The general algorithm leaves many aspects unspecified, such as how
record weights are initialized and updated.  In fact, several variants
of the algorithm are possible.  The choice of model weights in
step~\ref{step:model-w} is explained at the end of
Section~\ref{sec:error_bounds}.  In order to pick concrete values of
the weight update coefficients we are going to state three desirable
properties of an uplift booster and use those properties to derive
three different algorithms.  The parameters used by those algorithms
are summarized in Table~\ref{tab:boost-coeffs}.  The errors
$\epsilon^T_m$ and $\epsilon^C_m$ are defined in
Section~\ref{sec:uplift-err}.

\begin{table}[t]
  \caption{Coefficient values for the three proposed uplift boosting algorithms.}
  \label{tab:boost-coeffs}
  \begin{center}
  \begin{tabular}{|c|c|c|c|}
    \hline
    Algorithm                    & $\beta_m^T$                                         & $\beta_m^C$                        & weight initialization                                                                     \\\hline\hline
    uplift AdaBoost            & \multicolumn{2}{|c|}{$\frac{p^T_m\epsilon^T_m+p^C_m\epsilon^C_m}{1-p^T_m\epsilon^T_m-p^C_m\epsilon^C_m}$} & $w^T_{1,i},w^C_{1,i}\leftarrow 1$ \\\hline
    balanced uplift boosting & \multicolumn{2}{|c|}{See Theorem~\ref{thm:bub_solution}}  & $w^T_{1,i}\leftarrow\frac{1}{N^T}$; $w^C_{1,i}\leftarrow\frac{1}{N^C}$ \\\hline
    balanced forgeting boosting  & $\frac{\epsilon^C_m}{1-\epsilon^T_m}$   & $\frac{\epsilon^T_m}{1-\epsilon^C_m}$       & $w^T_{1,i}\leftarrow\frac{1}{N^T}$; $w^C_{1,i}\leftarrow\frac{1}{N^C}$ \\\hline
  \end{tabular}
  \end{center}
\end{table}

Note that the algorithm is a discrete boosting
algorithm~\cite{freund:97,schapire1999improved}, that is, the base
learners are assumed to return a discrete decision on whether the
action should be taken (1) or not (0).  Algorithm~\ref{alg:upl-boost},
as presented in the figure, also returns a decision.  However, it can
also return a numerical score,
\[
s(x) = \sum_{m=1}^M \left( \log\frac{1}{\beta_m} \right) h_m(x),
\]
indicating how likely it is that the effect of the action is positive
on a given case.  In the experimental Section~\ref{sec:exper} we will
use this variant of the algorithm.

AdaBoost can suffer from premature stops when the sum of weights of
misclassified cases becomes $0$ or is greater than $1/2$.  This
problem turns to be even more troublesome in the uplift modeling case.
Hence, in step~\ref{step:restart-weights} of
Algorithm~\ref{alg:upl-boost} we restart the algorithm by assigning
random weights drawn from the exponential distribution to records in
both training datasets.  The technique has been suggested for
classification boosting in~\cite{Webb2000}.

Before we proceed to describe the three desirable properties of uplift
boosting algorithms we need to define an uplift analogue of
classification error.

\subsection{An uplift analogue of classification error}\label{sec:uplift-err}

Due to the Fundamental Problem of Causal Inference we cannot tell
whether an uplift model correctly classified a given instance.  We
will, however, define an approximate notion of classification error in
the uplift case.  A record $(x_i^T,y_i^T)$ is assumed to be classified
correctly by an uplift model $h$ if $h(x_i^T)=y_i^T$ and
$(x_i^T,y_i^T)\in\dataset^T$; a record $(x_i^C,y_i^C)$ is assumed to be
classified correctly if $h(x_i^C)=1-y_i^C$ and
$(x_i^C,y_i^C)\in\dataset^C$.

Intuitively, if a record $(x_i^T,y_i^T)$ belongs to the treatment
group and a model $h$ predicts that it should receive the treatment
($h(x_i^T)=1$) then the outcome should be positive ($y_i^T=1$) if the
recommendation is to be correct.  Note that the gain from the action
might also be neutral if a success would have occurred also without
treatment, but at least the model's recommendation is not in
contradiction with the observed outcome.  If, on the contrary, the
outcome for a record in the treatment group is $0$ and $h(x_i^T)=1$,
the prediction is clearly wrong as the true effect of the action can
at best be neutral.

In the control group the situation is reversed.  If the outcome was
positive ($y_i^C=1$) but the model predicted that the treatment should
be applied ($h(x_i^C)=1$), the prediction is clearly wrong, since the
treatment cannot be truly beneficial, it can at best be neutral.  To
simplify notation we will introduce the following indicators:
\begin{align}
  e^T(x_i^T) & = 
  \begin{cases}
    0 & \mbox{if }x_i^T\in\dataset^T\mbox{ and }h(x_i^T)=y_i^T,\cr
    1 & \mbox{if }x_i^T\in\dataset^T\mbox{ and }h(x_i^T)\neq y_i^T,
  \end{cases}\label{eq:upl-err1}\\
  e^C(x_i^C) & = 
  \begin{cases}
    0 & \mbox{if }x_i^C\in\dataset^C\mbox{ and }h(x_i^C)\neq y_i^C,\cr
    1 & \mbox{if }x_i^C\in\dataset^C\mbox{ and }h(x_i^C)=y_i^C.\label{eq:upl-err2}
  \end{cases}
\end{align}
An index $m$ will be added to indicate the $m$-th iteration of the
algorithm.  Let us now define uplift analogues of classification error
on the treatment and control datasets and a combined error:
\begin{equation}
  \epsilon_m^T=\frac{\sum_{i:\;e_m^T(x_i)=1}w^T_{m,i}}{\sum_{i=1}^{N^T}w^T_{m,i}},\quad
  \epsilon_m^C=\frac{\sum_{i:\;e_m^C(x_i)=1}w^C_{m,i}}{\sum_{i=1}^{N^C}w^C_{m,i}},\quad
  \epsilon_m= p_m^T\epsilon_m^T + p_m^C\epsilon_m^C.\label{eq:err-upl}
\end{equation}
The sums above are a shorthand notation for summing over misclassified
instances in the treatment and control training sets, which will also
be used later in the paper.

In~\cite{my:uplift-clinical-ml} a class variable transformation was
introduced which replaces class values $y_i^C$ in the control group
with their reverses $1-y_i^C$ while keeping the treatment set class
values unchanged.  It is easy to see that the errors defined in
Equation~\ref{eq:err-upl} are equivalent to standard classification
errors for the transformed class.  According
to~\cite{my:uplift-clinical-ml} the method was not very successful in
regression settings, but as we will show later in this paper, it can
play a useful role in developing uplift boosting algorithms.

We now proceed to introduce the three desirable properties of uplift
boosting algorithms.

\subsection{Balance}

An uplift boosting algorithm is said to satisfy the {\em balance
  condition} if at each\linebreak[4] iteration~$m$
\begin{equation}\label{eq:bub_equal_masses}
\sum_{i=1}^{N^T} w_{m,i}^T = \sum_{i=1}^{N^C} w_{m,i}^C.
\end{equation}
Due to weight normalization in step~\ref{step:norm} the condition can
equivalently be stated as $p_m^T=p^C_m=\frac{1}{2}$ for all $m$'s.  In
other words, the total weights assigned to the treatment and control
groups remain constant and equal.

The reason balance property may be desirable is to prevent a situation
when one of the groups gains much higher weight than the other and, as
a result, weak learners focus only on the treatment or control group
instead of the differences between them.

Let us now inductively rephrase the condition in terms of an equation
on $\beta_m^T$ and $\beta_m^C$.  For $m=1$ the balance condition can
be achieved by proper initialization of weights.  Suppose now that it
is satisfied at iteration $m$.  From Equation (\ref{eq:bub_equal_masses})
and steps~\ref{step:upd-t} and~\ref{step:upd-c} of
Algorithm~\ref{alg:upl-boost} we get, after splitting the sums over
correctly and incorrectly classified cases, that to ensure the
condition holds at iteration $m+1$, the following equation must be true:
\begin{equation*}
\sum_{i:\;e_m^T(x_i)=1} w_{m,i}^T + \beta_m^T\sum_{i:\;e_m^T(x_i)=0} w_{m,i}^T = \sum_{i:\;e_m^C(x_i)=1} w_{m,i}^C + \beta_m^C\sum_{i:\;e_m^C(x_i)=0} w_{m,i}^C.
\end{equation*}
Dividing both sides by $\sum_{i=1}^{N^T} w_{m,i}^T = \sum_{i=1}^{N^C}
w_{m,i}^C$ and using Equation~\ref{eq:err-upl} we get
\begin{equation*}
\epsilon_m^T + \left( 1 - \epsilon_m^T \right) \beta_m^T = \epsilon_m^C + \left( 1 - \epsilon_m^C \right) \beta_m^C,
\end{equation*}
and
\begin{equation}\label{eq:bub_notation_full}
\beta_m^T = \frac{\epsilon_m^C - \epsilon_m^T}{1 - \epsilon_m^T} + \frac{1 - \epsilon_m^C}{1 - \epsilon_m^T}\beta_m^C.
\end{equation}
Equation~\ref{eq:bub_notation_full} holds when $\beta_m^T, \beta_m^C
> 0$ and both errors are not equal to $0$ nor $1$ (there must be good
and bad cases in both datasets).

\subsection{Forgetting the last ensemble member}\label{sec:forgetting}

Let us now examine what the property of forgetting the last model
added to the ensemble means in the context of uplift error defined in
Equation~\ref{eq:err-upl}.  To forget the member $h_m$ added in iteration
$m$ we need to choose weights in iteration $m+1$ such that the combined
error of $h_m$ is exactly one half, $\epsilon_m = \frac{1}{2}$.  Using
a derivation similar to that for the balance condition we see that
$\beta_m^T$ and $\beta_m^C$ need to satisfy the condition
\begin{equation}\label{eq:forget-w}
\beta_m^T\sum_{i:\;e_m^T(x_i)=0} w_{m,i}^T + \beta_m^C\sum_{i:\;e_m^C(x_i)=0} w_{m,i}^C = \sum_{i:\;e_m^T(x_i)=1} w_{m,i}^T + \sum_{i:\;e_m^C(x_i)=1} w_{m,i}^C,
\end{equation}
that is, the total new weights of correctly classified examples need to
be equal to total new weights of incorrectly classified
examples. After dividing both sides by $\sum_{i=1}^{N^T} w_{m,i}^T +
\sum_{i=1}^{N^C} w_{m,i}^C$ the equation becomes
\begin{equation}\label{eq:forget-e}
p^T_m(1-\epsilon_m^T)\beta_m^T + p^C_m(1-\epsilon_m^C)\beta_m^C = p^T_m\epsilon_m^T + p^C_m\epsilon_m^C.
\end{equation}
Note that unlike classical boosting, this condition does not uniquely
determine record weights.  Let us now give a justification of this
condition in terms of performance of an uplift model.

\begin{theorem}\label{thm:forgetting-interp}
Let $h$ be an uplift model.  If the balance condition holds and the
assignment of cases to the treatment and control groups is random then
the condition that the combined uplift error $\epsilon$ be equal to
$\frac{1}{2}$ is equivalent to
\begin{multline}\label{eq:forg-interp}
P(h=1) \left[P^T(y=1|h=1) - P^C(y=1|h=1)\right]\\ + P(h=0)\left[P^C(y=1|h=0)-P^T(y=1|h=0)\right] = 0.
\end{multline}
\end{theorem}
The proof can be found in Appendix A.
Note that the left term in~(\ref{eq:forg-interp}) is the total gain in
success probability due to the action being taken on cases selected by
the model and the right term is the gain from not taking the action on
cases not selected by the model.  A good uplift model tries to
maximize both quantities, so the sum being equal to zero corresponds
to a model giving no overall gain over the controls.

When the balance condition holds, the forgetting property thus has a
clear interpretation in terms of uplift model performance.  When the balance
condition does not hold, the interpretation is, at least partially,
lost.

\subsection{Training error bound}\label{sec:error_bounds}

We now derive an upper bound on the uplift analogue to the training
error of an ensemble with $M$ members.  It is desirable for an uplift
algorithm to achieve the fastest possible decrease of this upper
bound, or at least to guarantee that the error does not increase as
more models are added.
We begin by defining the analogue of Equation~\ref{eq:err-upl} for
the whole ensemble.  Let $e_f^T(x_i^T)$ and $e_f^C(x_i^C)$ denote the
indicators defined in Equations~\ref{eq:upl-err1}
and~\ref{eq:upl-err2} for the final ensemble $h_f$ defined in
Equation~\ref{eq:final_hypothesis}.  The errors defined in
Equation~\ref{eq:err-upl} for $h_f$ become
\begin{equation}
  \epsilon_f^T=\frac{\sum_{i:\;e_f^T(x_i)=1}w^T_{1,i}}{\sum_{i=1}^{N^T}w^T_{1,i}},\quad
  \epsilon_f^C=\frac{\sum_{i:\;e_f^C(x_i)=1}w^C_{1,i}}{\sum_{i=1}^{N^C}w^C_{1,i}},\quad
  \epsilon_f= p_1^T\epsilon_f^T + p_1^C\epsilon_f^C.\label{eq:err-f}
\end{equation}

Following~\cite{freund:97} in this derivation we ignore the weight
normalization step~\ref{step:norm} of the algorithm since it does not
influence the way errors (and consequently scaling factors) are
calculated.  It only affects step~\ref{step:build-mod}, which has no
effect on the derivations below.

Following the derivation from~\cite{freund:97} separately for the
treatment and control datasets, we obtain
\begin{align*}
\sum_{i=1}^{N^T}& w_{m+1,i}^T + \sum_{i=1}^{N^C} w_{m+1,i}^C \\ 
&\leq\left( \sum_{i=1}^{N^T} w_{m,i}^T \right) (1 - (1 - \epsilon_m^T)(1 - \beta_m^T)) + \left( \sum_{i=1}^{N^C} w_{m,i}^C \right) (1 - (1 - \epsilon_m^C)(1 - \beta_m^C)) \\
&=\sum_{i=1}^{N^T} w_{m,i}^T + \sum_{i=1}^{N^C} w_{m,i}^C - (1 - \epsilon_m^T)(1 - \beta_m^T) \sum_{i=1}^{N^T} w_{m,i}^T - (1 - \epsilon_m^C)(1 - \beta_m^C) \sum_{i=1}^{N^C} w_{m,i}^C\\
&\leq \left(\sum_{i=1}^{N^T} w_{m,i}^T + \sum_{i=1}^{N^C} w_{m,i}^C \right) \left[ 1 - p_m^T (1 - \epsilon_m^T)(1 - \beta_m^T) - p_m^C (1 - \epsilon_m^C)(1 - \beta_m^C) \right],
\end{align*}
where the last inequality uses the notation introduced in Equation~\ref{eq:def-ptpc}.
Applying the inequality $M$ times we get:
\begin{multline}\label{eq:schapire-combinded}
\sum_{i=1}^{N^T} w_{M+1,i}^T + \sum_{i=1}^{N^C} w_{M+1,i}^C\\ \leq \left(\sum_{i=1}^{N^T} w_{1,i}^T + \sum_{i=1}^{N^C} w_{1,i}^C\right)\prod_{m=1}^M \left[ 1 - p_m^T (1 - \epsilon_m^T)(1 - \beta_m^T) - p_m^C (1 - \epsilon_m^C)(1 - \beta_m^C) \right].
\end{multline}
In order to proceed further, we need to make an additional assumption:
\begin{equation*}
\prod_{m=1}^M {\beta_m^T} ^ {1-e_m^T(x_i^T)} \geq \prod_{m=1}^M {\beta_m} ^ {1-e_m^T(x_i^T)}
\quad\mbox{and}\quad
\prod_{m=1}^M {\beta_m^C} ^ {1-e_m^C(x_i^C)} \geq \prod_{m=1}^M {\beta_m} ^ {1-e_m^C(x_i^C)},
\end{equation*}
respectively for all $x_i^T$, $x_i^C$.  Recall that $\beta_m$ is the
weight of $m$-th ensemble member and $\beta_m^T,\beta_m^C$ the scaling
factors for record weights.  Notice that this assumption is implied by
\begin{equation}
\beta_m^T \geq \beta_m, \;\;\;\;\; \beta_m^C \geq\beta_m,\label{cnd:ub-betas}
\end{equation}
which is stronger, but easier to use in practice.

The final boosting ensemble with $M$ members makes a mistake on $x_i^T
\in \dataset^T$, the $i$-th training instance of the treatment dataset
(respectively the $i$-th control instance, $x_i^C \in \dataset^C$) only if~\cite{freund:97}
\begin{equation}
\prod_{m=1}^M \left( \beta_m \right)^{-e_m^T(x_i^T)} \geq \left(
\prod_{m=1}^M \beta_m \right)^{-\frac{1}{2}} \qquad\mbox{and}\qquad
\prod_{m=1}^M \left( \beta_m \right)^{-e_m^C(x_i^C)} \geq \left(
\prod_{m=1}^M \beta_m \right)^{-\frac{1}{2}}.\label{eq:ensemb-err}
\end{equation}
Notice also that the
final weights of an instance $x_i \in \dataset^T$ and $x_i \in \dataset^C$ are,
respectively,
$$w_{M+1,i}^T = w_{1,i}^T \prod_{m=1}^M \left( \beta_m^T
\right)^{1-e_m^T(x_i^T)} \quad\mbox{and}\quad w_{M+1,i}^C = w_{1,i}^C
\prod_{m=1}^M \left( \beta_m^C \right)^{1-e_m^C(x_i^C)},$$ where
$w_{1,i}^T$ and $w_{1,i}^C$ are initial weights set in
step~\ref{step:init} of the algorithm.
Now we can combine the above equations to bound the error $\epsilon_f=p_1^T\epsilon_f^T + p_1^C\epsilon_f^C$:
\begin{align*}
\sum_{i=1}^{N^T} w_{M+1,i}^T& + \sum_{i=1}^{N^C} w_{M+1,i}^C \geq \sum_{i:\; e_f^T(x_i) \neq 0} w_{M+1,i}^T + \sum_{i:\; e_f^C(x_i) \neq 0}w_{M+1,i}^C \\
&=\sum_{i:\; e_f^T(x_i) \neq 0} w_{1,i}^T \prod_{m=1}^M \left( \beta_m^T \right)^{1-e_m^T(x_i)} + \sum_{i:\; e_f^C(x_i) \neq 0}w_{1,i}^C \prod_{m=1}^M \left( \beta_m^C \right)^{1-e_m^C(x_i)} \\
&\geq\sum_{i:\; e_f^T(x_i) \neq 0} w_{1,i}^T \prod_{m=1}^M \left( \beta_m \right)^{1-e_m^T(x_i)} + \sum_{i:\; e_f^C(x_i) \neq 0} w_{1,i}^C \prod_{m=1}^M \left( \beta_m \right)^{1-e_m^C(x_i)} \\
&\geq\left(\prod_{m=1}^M \beta_m \right)^{\frac{1}{2}} \left(\sum_{i:\; e_f^T(x_i) \neq 0} w_{1,i}^T  + \sum_{i:\; e_f^C(x_i) \neq 0} w_{1,i}^C \right)
\end{align*}
The second inequality follows from Equation~\ref{cnd:ub-betas} and
the last one was obtained using Equation~\ref{eq:ensemb-err}.  Combining with
Inequality~\ref{eq:schapire-combinded} we obtain
\begin{equation}\label{eq:ub-err-bound}
\prod_{m=1}^M \frac{ 1 - p_m^T (1 - \epsilon_m^T)(1 - \beta_m^T) - p_m^C (1 - \epsilon_m^C)(1 - \beta_m^C) }{\left( \beta_m \right)^{\frac{1}{2}}} \geq p_{1}^T\epsilon_f^T + p_{1}^C\epsilon_f^C = \epsilon_f,
\end{equation}
an upper bound for resubstitution error of the final hypothesis on
datasets $\dataset^T$, $\dataset^C$.

Note that for fixed $\beta_m^T$ and $\beta_m^C$, the bound is
optimized (subject to the constraint given in
Equation~\ref{cnd:ub-betas}) by choosing
$\beta_m=\min\{\beta_m^T,\beta_m^C\}$.  This is the rationale for
step~\ref{step:model-w} of Algorithm~\ref{alg:upl-boost}.

\section{Uplift boosting algorithms}

In the previous section we introduced three properties which uplift
boosting algorithms should satisfy.  Unfortunately, as will soon
become apparent, all three cannot be satisfied by a single algorithm.
Therefore, in this section we will derive three uplift boosting
algorithms, each satisfying two of the properties.

\subsection{Uplift AdaBoost}\label{sec:upl-adaboost}

We begin with {uplift AdaBoost}, an algorithm obtained by
optimizing the upper bound on the training error given in
Equation~\ref{eq:ub-err-bound}.  The optimization will proceed
iteratively for $m=1,\ldots,M$; in the $m$-th iteration we minimize the
expression
\begin{equation}\label{ex:ub-iter-opt-task}
\frac{ 1 - p_m^T (1 - \epsilon_m^T)(1 - \beta_m^T) - p_m^C (1 - \epsilon_m^C)(1 - \beta_m^C) }{\sqrt{\beta_m}}
\end{equation}
over $\beta_m^T,\beta_m^C$ with $\epsilon_m^T,\epsilon_m^C$ fixed.
Notice that, while minimizing this expression, we need to respect the
constraint given in Equation~\ref{cnd:ub-betas}.  Suppose the
expression is minimized by some $\beta_m^{T*}$, $\beta_m^{C*}$,
$\beta_m^*$, with $\beta_m^{*T} > \beta_m^*$.  However, the numerator
of the fraction above is an increasing function of $\beta_m^T$ and
$\beta_m^C$ for $\beta_m$ fixed.  Therefore we could replace
$\beta_m^{T*}$ with $\beta_m^*$ to obtain a better solution.
Analogous argument holds for $\beta_m^C$.  We can conclude that an optimal
solution must satisfy $\beta_m^{T*} = \beta_m^{C*} = \beta_m^*$,
that is, treatment, control and model weights need to be equal.  Taking
this equality into account, Equation~\ref{ex:ub-iter-opt-task} can be
simplified to
\begin{equation}\label{ex:ub-iter-opt-task-simply}
\frac{1 - \left[ 1 - (p_m^T\epsilon_m^T + p_m^C\epsilon_m^C) \right](1 - \beta_m)}{\sqrt{\beta_m}},
\end{equation}
very similar to the one obtained by \cite{freund:97}.  Taking the
derivative and equating to zero we obtain the optimal $\beta_m$:
\begin{equation}\label{eq:ub-beta-optim}
\beta_m^* = \frac{p_m^T\epsilon_m^T + p_m^C\epsilon_m^C}{1 - \left( p_m^T\epsilon_m^T + p_m^C\epsilon_m^C \right)}.
\end{equation}
This result is identical to classical boosting with the classification
error replaced by its uplift analogue.  In fact, the algorithm is
identical to an application of AdaBoost with class variable
transformation from~\cite{my:uplift-clinical-ml}.  Here we prove that
it is in a certain sense optimal.

At each iteration the upper bound is multiplied by
$$2\sqrt{(p_m^T\epsilon_m^T + p_m^C\epsilon_m^C)(1 - \left(
  p_m^T\epsilon_m^T + p_m^C\epsilon_m^C \right))}.$$  It is easy to see
that the bound cannot increase as more members are added to the
ensemble and does in fact decrease, as long as $p_m^T\epsilon_m^T +
p_m^C\epsilon_m^C\neq\frac{1}{2}$.

Let us now look into the two remaining properties.  Using
$\beta_m^T=\beta_m^C=\beta_m$, Equation~\ref{eq:forget-e} can be
rewritten as
\[
1-(p^T_m\epsilon_m^T +p^C_m\epsilon_m^C) = \frac{p^T_m\epsilon_m^T + p^C_m\epsilon_m^C}{\beta_m}.
\]
Substituting Equation~\ref{eq:ub-beta-optim} proves that the
forgetting property is satisfied by this algorithm.

Note, however, that the balance property need not be
satisfied\footnote{This is the case only if $N^T = N^C$ and
  $\epsilon_m^T=\epsilon_m^C$ for all $m$.}.  There is therefore a
risk that the method may focus on only one of the groups: treatment or
control, reducing the weights of records in the other group to
insignificance.  Moreover, the uplift interpretation of the forgetting
property given in Theorem~\ref{thm:forgetting-interp} is no longer
valid.

\subsection{Balanced uplift boosting}

Uplift AdaBoost presented in Section~\ref{sec:upl-adaboost} may suffer
from imbalance between the treatment and control training sets.  In
this section we will develop the balanced uplift boosting algorithm
which optimizes the error bound given in
Equation~\ref{eq:ub-err-bound} under the constraint that the balance
condition holds.  Let us first define constants $a_m$ and $b_m$ to
write Equation~\ref{eq:bub_notation_full} in a more concise form:
\begin{equation}\label{eq:bub_notation}
\beta_m^T = \underbrace{\frac{\epsilon_m^C - \epsilon_m^T}{1 - \epsilon_m^T}}_{b_m} + \underbrace{\frac{1 - \epsilon_m^C}{1 - \epsilon_m^T}}_{a_m} \beta_m^C = a_m\beta_m^C + b_m.
\end{equation}
Substituting into Equation~\ref{ex:ub-iter-opt-task} we get a new
bound on error decrease in iteration $m$ which takes into account the
balance condition:
\begin{equation*}
f\left(\beta_m^C\right) = \frac{ 1 - (1 - \epsilon_m^C)(1 - \beta_m^C)}{\sqrt{\beta_m}}.
\end{equation*}
Using the assumption $\beta_m \leq \min\{\beta_m^T, \beta_m^C\}$ given
in Equation~\ref{cnd:ub-betas} and the fact that $f$ is a decreasing
function of $\beta_m$ we get that, at optimum, $\beta_m =
\min\{\beta_m^T, \beta_m^C\}$.  The bound becomes
\begin{equation}\label{ex:bub_iter_opt_task}
f\left(\beta_m^C\right) = \frac{ 1 - (1 - \epsilon_m^C)(1 - \beta_m^C)}{\sqrt{\min\{\beta_m^C,a_m\beta_m^C + b_m\}}}.
\end{equation}
Minimizing this bound we obtain the value of $\beta_m^C$; $\beta_m^T$
is then computed from Equation~\ref{eq:bub_notation}.  The optimal
value of $\beta_m^C$ is given by the following theorem.
\begin{theorem}\label{thm:bub_solution}
The bound in Equation~\ref{ex:bub_iter_opt_task} is minimized for
\begin{equation}\label{eq:bub_argmin}
\argmin_{\beta_m^C \in (0,+\infty)} f(\beta_m^C)=\left\{
\begin{array}{ll}
\frac{2\epsilon_m^T-\epsilon_m^C}{1-\epsilon_m^C}& if \; \epsilon_m^C < \epsilon_m^T < \frac{1}{2} \;or\; \frac{1}{2} < \epsilon_m^T < \epsilon_m^C, \\
\frac{\epsilon_m^C}{1-\epsilon_m^C}& if \; \epsilon_m^T < \epsilon_m^C < \frac{1}{2} \;or\; \frac{1}{2} < \epsilon_m^C < \epsilon_m^T, \\
1& otherwise,
\end{array}
\right.
\end{equation}
and the corresponding minimum is
\begin{equation}\label{eq:bub_min}
\min_{\beta_m^C \in (0,+\infty)} f(\beta_m^C)=\left\{
\begin{array}{ll}
2\sqrt{\epsilon_m^T(1-\epsilon_m^T)}& if \; \epsilon_m^C < \epsilon_m^T < \frac{1}{2} \;or\; \frac{1}{2} < \epsilon_m^T < \epsilon_m^C, \\
2\sqrt{\epsilon_m^C(1-\epsilon_m^C)}& if \; \epsilon_m^T < \epsilon_m^C < \frac{1}{2} \;or\; \frac{1}{2} < \epsilon_m^C < \epsilon_m^T, \\
1& otherwise.
\end{array}
\right.
\end{equation}
\end{theorem}
See Appendix B for the proof.
This result is easy to interpret.  Suppose that $\epsilon_m^T <
\epsilon_m^C < \frac{1}{2}$.  We take the larger of the two training
errors, $\epsilon^C_m$, reweight the control dataset using classical
boosting weight update and reweight the treatment dataset such that
the balance condition is maintained.  If $\epsilon_m^C < \epsilon_m^T
< \frac{1}{2}$ the situation is symmetrical.
The decrease in upper bound (Equation~\ref{eq:bub_min}) is interpreted
as follows: compute the decrease rate of standard AdaBoost separately
based on the treatment and control errors.  Then take the larger
(i.e.~worse) of the two factors.

Note that we are guaranteed that the bound will actually decrease only
when $\epsilon_m^T,\epsilon_m^C$ lie on the same side of
$\frac{1}{2}$, a much stricter condition than for uplift AdaBoost.  In
any case we are, however, guaranteed that the bound will not increase
when another model is added to the ensemble.

Note also, that balanced uplift boosting does not, in general, have
the property of forgetting the last member added to the ensemble.
This is an immediate consequence of the derivations in the next section.

\subsection{Balanced forgetting boosting}\label{sec:FBW}

We now present another uplift boosting algorithm satisfying different
properties, namely the balance condition and forgetting of the last
added member.  We assume that at the $m$-th iteration the balance
condition is satisfied, $p^T_m=p^C_m=\frac{1}{2}$.  We want to
pick factors $\beta^T_m$, $\beta^C_m$ such that in the next iteration
the balance condition continues to hold and the current model is
`forgotten', that is, its error is exactly one half.  Since
$p^T_m=p^C_m=\frac{1}{2}$, Equation~\ref{eq:forget-e} can be rewritten
as
\begin{equation}\label{eq:forgetting}
\beta_m^T = \frac{\epsilon_m^T + \epsilon_m^C}{1-\epsilon_m^T} - \frac{1-\epsilon_m^C}{1-\epsilon_m^T}\beta_m^C,
\end{equation}
which, after adding to Equation~\ref{eq:bub_notation_full}, gives
$\beta_m^T = \frac{\epsilon_m^C}{1-\epsilon_m^T}$. Analogously,
$\beta_m^C = \frac{\epsilon_m^T}{1-\epsilon_m^C}$.  We still need to
choose $\beta_m$.  Since the bound in
Equation~\ref{ex:ub-iter-opt-task} is a decreasing function of
$\beta_m$ we set it to $\min\{\beta_m^T, \beta_m^C\}$, the highest
value satisfying Inequality~\ref{cnd:ub-betas}.  Note that the upper
bound on the error need no longer be optimal; let us compute it
explicitly.

First, we provide a condition for $\beta_m^T < \beta_m^C$ or equivalently
$\frac{\epsilon_m^C}{1-\epsilon_m^T}<\frac{\epsilon_m^T}{1-\epsilon_m^C}$.
After multiplying, the second inequality becomes
$\epsilon_m^C(1-\epsilon_m^C)<\epsilon_m^T(1-\epsilon_m^T)$, which is
true when $\epsilon_m^C$ is further from $\frac{1}{2}$ than $\epsilon_m^T$.
Assume now $\epsilon_m^C<\epsilon_m^T<\frac{1}{2}$, so $\beta_m^T <
\beta_m^C$; the remaining cases are analogous and are thus omitted.
After substitution of the expressions for $\beta_m^T$ and $\beta_m^C$
into Equation~\ref{ex:ub-iter-opt-task} we get
\begin{align*}
\frac{ 1 - p_m^T (1 - \epsilon_m^T)(1 -
  \frac{\epsilon_m^C}{1-\epsilon_m^T}) - p_m^C (1 - \epsilon_m^C)(1 -
  \frac{\epsilon_m^T}{1-\epsilon_m^C})
}{\sqrt{\frac{\epsilon_m^C}{1-\epsilon_m^T}}}
=\sqrt{\frac{1-\epsilon_m^T}{\epsilon_m^C}}(\epsilon_m^T+\epsilon_m^C).
\end{align*}
 Notice that the balanced
forgetting boosting algorithm does not guarantee that the error bound
will decrease at all: with $\epsilon^T$ constant and $\epsilon^C$
tending to zero, the bound becomes infinite.  It is thus possible that
the algorithm will diverge when more members are added to the
ensemble.


Let us now compare the bound with that of the balanced uplift boosting
algorithm.  When $\epsilon_m^C<\epsilon_m^T<\frac{1}{2}$ its error
bound is $2\sqrt{\epsilon_m^T(1-\epsilon_m^T)}$.  From
$(\sqrt{\epsilon_m^T}-\sqrt{\epsilon_m^C})^2\geq 0$ it follows that
$\epsilon_m^T+\epsilon_m^C\geq
2\sqrt{\epsilon_m^T}\sqrt{\epsilon_m^C}$ giving
\[
\sqrt{1-\epsilon_m^T}\frac{\epsilon_m^T+\epsilon_m^C}{\sqrt{\epsilon_m^C}}\geq
\sqrt{1-\epsilon_m^T}\frac{2\sqrt{\epsilon_m^T}\sqrt{\epsilon_m^C}}{\sqrt{\epsilon_m^C}}
=2\sqrt{(1-\epsilon_m^T)\epsilon_m^T}.
\]
The other cases for which $\epsilon_m^T,\epsilon_m^C$ lie on the same
side of $\frac{1}{2}$ are analogous.  In those cases balanced uplift
boosting always gives stronger guarantees on the decrease of the
training set error.  If, however, $\epsilon_m^T,\epsilon_m^C$ lie on
the opposite sides of $\frac{1}{2}$, the comparison can go either way.

\section{Experimental evaluation}\label{sec:exper}

In this section we present an experimental evaluation of the three
proposed algorithms and compare their performance with performance of
the base models and bagging.  We begin by describing the test datasets
we are going to use, then review the approaches to evaluating
uplift models and finally present the experimental results.

\subsection{Benchmark data sets}

Unfortunately, not many datasets involving true control groups obtained
through randomized experiments are publicly available.  We have used
one marketing dataset and a few datasets from controlled medical
trials, see Table~\ref{tab:real-datasets} for a summary.

\begin{table}
  \caption{Datasets from randomized trials used in the paper.}\label{tab:real-datasets}
  \begin{center}
    \begin{tabular}{|l|l|c|c|c|c|}
      \hline
                       &                                 & \multicolumn{3}{c|}{\#records}                    & \#attri- \\\cline{3-5}
      dataset          & source                          & treatment                      & control & total  & butes     \\\hline
      Hillstrom visit  & MineThatData blog               & 21,306                         & 21,306  & 42,612 & 8         \\
                       & \cite{Hillstrom2008}           &                                &         &        &           \\
      burn             & {\tt R}, {\tt KMsurv} package   & 84                             & 70      & 154    & 17        \\
      hodg             & {\tt R}, {\tt KMsurv} package   & 16                             & 27      & 43     & 7         \\
      bladder          & {\tt R}, {\tt survival} package & 38                             & 47      & 85     & 6         \\
      colon death      & {\tt R}, {\tt survival} package & 614                            & 315     & 929    & 14        \\
      colon recurrence & {\tt R}, {\tt survival} package & 614                            & 315     & 929    & 14        \\
      veteran          & {\tt R}, {\tt survival} package & 69                             & 68      & 137    & 9         \\
      \hline
    \end{tabular}
  \end{center}
\end{table}

The marketing dataset comes from Kevin Hillstrom's MineThatData
blog~\cite{Hillstrom2008} and contains results of an e-mail campaign
for an Internet retailer.  The dataset contains information about
64,000 customers who have been randomly split into three groups: the
first received an e-mail campaign advertising men's merchandise, the
second a campaign advertising women's merchandise, and the third was
kept as a control.  Data is available on whether a person visited the
website and/or made a purchase (conversion).  We only use data on
visits since very few conversions actually occurred.  In this paper we
use only the treatment group which received women's merchandise
campaign because this group was, overall, much more responsive.

The six medical datasets, based on real randomized trials, come from
the {\tt survival} and {\tt kmsurv} packages in the {\tt R}
statistical system.  Their descriptions are readily available online
and are thus omitted.  The {\tt colon} dataset involves two possible
treatments and a control.  We have kept only one of the treatments
(levamisole), the other (levamisole+5fu) was beneficial for all
patients and targeting specific groups brought no improvement.  All
those datasets contain survival data and require preprocessing before
uplift boosting can be applied.  \cite{my:upl-surv} proposed a simple
conversion method where some threshold $\theta$ is chosen and cases
for which the {\em observed} (possibly censored) survival time is at
least $\theta$ are considered positive outcomes, the remaining ones
negative.  The authors demonstrate that, under reasonable assumptions,
an uplift model trained on such data will indeed make correct
recommendations even though censoring is ignored.  Here we pick the
threshold $\theta$ to be the median observed survival time in the
dataset.

\subsection{Methodology}

In this section we discuss the evaluation of uplift models which is
more difficult than evaluating traditional classifiers due to the
Fundamental Problem of Causal Inference mentioned in
Section~\ref{sec:intro}.  For a given individual we only know the
outcome after treatment or without treatment, never both.  As a
result, we never know whether the action taken on a given individual
has been truly beneficial.  Therefore, we cannot assess model
performance at the level of single objects, this is possible only for
groups of objects.  Usually such groups are compared based on an
assumption that objects which are assigned similar scores by an uplift
model are indeed similar and can be compared with each other.

To assess performance of uplift models we will use so called {\em
  uplift curves}~\cite{RadcliffeTechRep,my:uplift-trees}.  Notice
that since building uplift models requires two training sets, we now
also have two test sets: treatment and control.  Recall that one of
the tools for assessing performance of standard classifiers are lift
curves\footnote{Also known as cumulative gains curves or cumulative
  accuracy profiles.}, where the $x$ axis corresponds to the number of
cases subjected to an action and the $y$ axis to the number of
successes captured.  In order to obtain an {\em uplift curve} we score
both test sets using the uplift model and subtract the lift curve
generated on the control test set from the lift curve generated on the
treatment test set.  The number of successes for both curves is
expressed as percentage of the total population such that the
subtraction is meaningful.

The interpretation of an uplift curve is as follows: on the $x$ axis
we select the percentage of the population on which the action is
performed and on the $y$ axis we read the net gain in success
probability (with respect to taking no action) achieved on the
targeted group.  The point at $x=100\%$ gives the gain that would have
been obtained if the action was applied to the whole population.  The
diagonal corresponds to targeting a randomly selected subset.  Example
curves are shown in Section~\ref{sec:results}, more details can be
found in~\cite{my:uplift-trees,RadcliffeTechRep}.

As with ROC curves, we can use the Area Under the Uplift Curve (AUUC)
to summarize model performance with a single number.  In this paper we
subtract the area under the diagonal from this value to obtain more
meaningful numbers.  Note that AUUC can be less than zero; this
happens when the model gives high scores to cases for which the action
has a predominantly negative effect.

The experiments have been performed by randomly splitting both
treatment and control datasets into train (80\%) and test (20\%)
parts.  The process has been repeated 256 times and the results
averaged to improve repeatability of the experiments.

\subsection{Experiments}\label{sec:results}

We will now present the experimental results.  As base models we are
going to use two types of unpruned E-divergence based uplift trees:
stumps (trees of height one) and trees of height three with eight
leaves (assuming binary splits).  See Section~\ref{sec:intro}
and~\cite{my:uplift-trees} for more details.

Recall from Section~\ref{sec:general-uplift-boosting} that our
boosting algorithms are discrete, that is ensemble members return 0-1
predictions\footnote{For this reason results for bagging are not
  comparable with those in~\cite{my:uplEnsembles}.}, not numerical
scores, however, bagging and boosting algorithms themselves do produce
numerical scores.  Since the base models can, in principle, also
return numerical scores (an estimate of the difference in success
probabilities between treatment and control),
Figures~\ref{fig:auucs-stump} and~\ref{fig:auucs-e8} include AUUCs for
base models in both 0-1 and score modalities.

We begin the experimental evaluation by showing some example uplift
curves for the three proposed uplift boosting algorithms.
Figure~\ref{fig:veteran} shows the curves for boosted E-divergence
based stumps on the {\tt veteran} dataset and
Figure~\ref{fig:colon-death} the curves for boosted three level
E-divergence based uplift trees on the {\tt colon death} dataset.

\begin{figure}[t]
  \begin{center}
    {\footnotesize
    \begin{tabular}{ccc}
      ~~uplift AdaBoost & ~~balanced uplift boosting & ~~balanced forgetting boosting\\
      \includegraphics[width=0.3\textwidth]{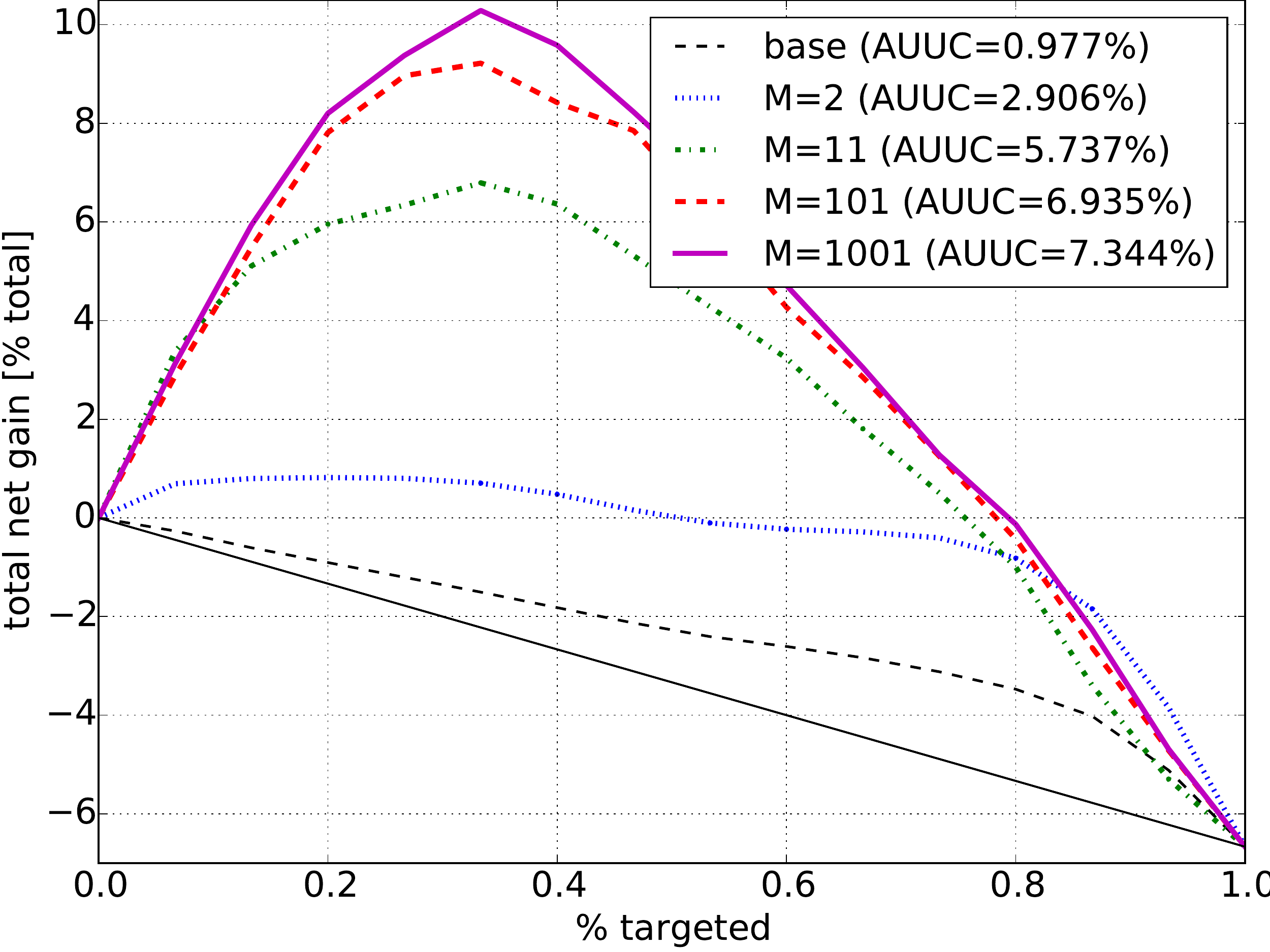}&
      \includegraphics[width=0.3\textwidth]{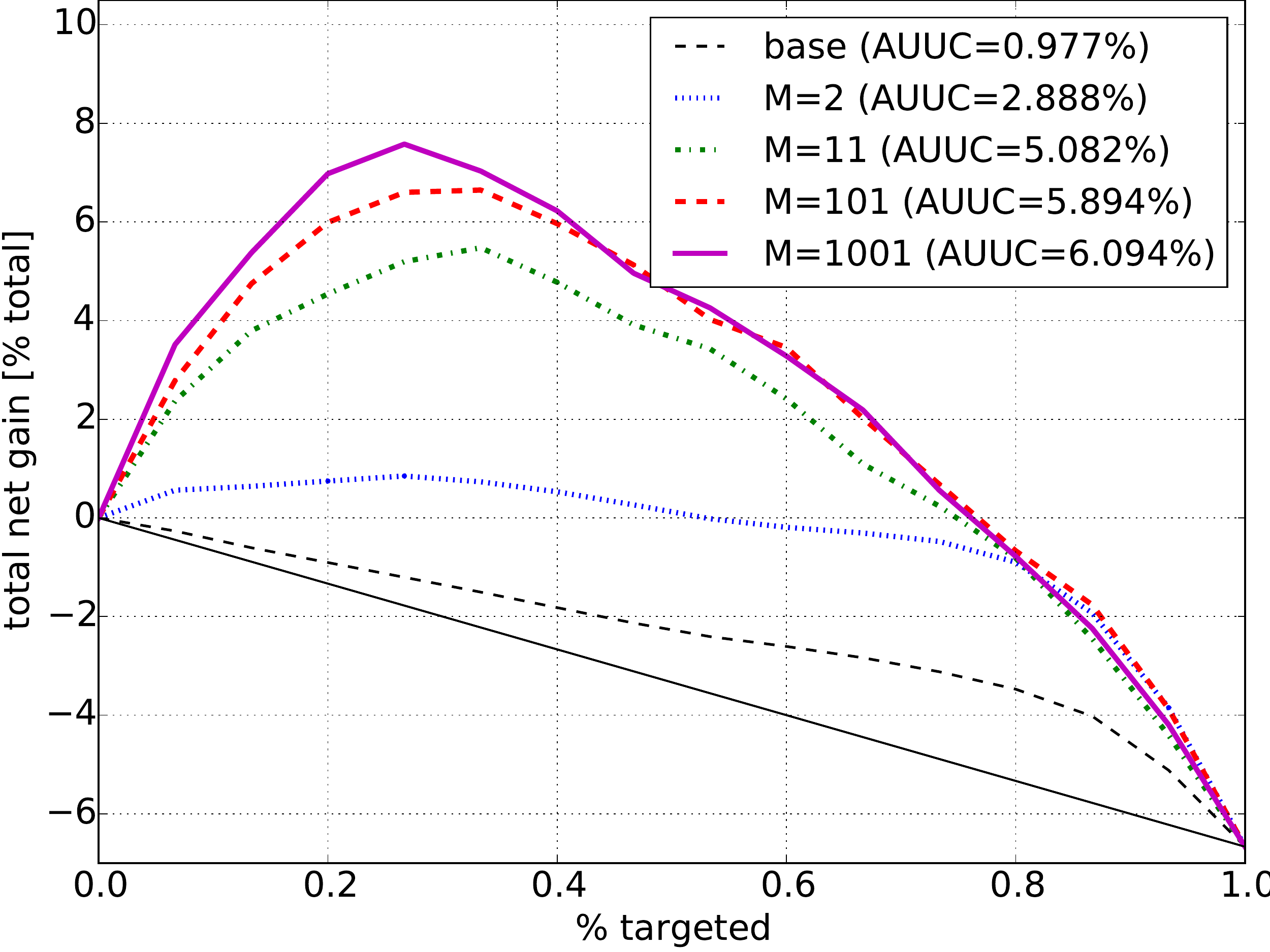}&
      \includegraphics[width=0.3\textwidth]{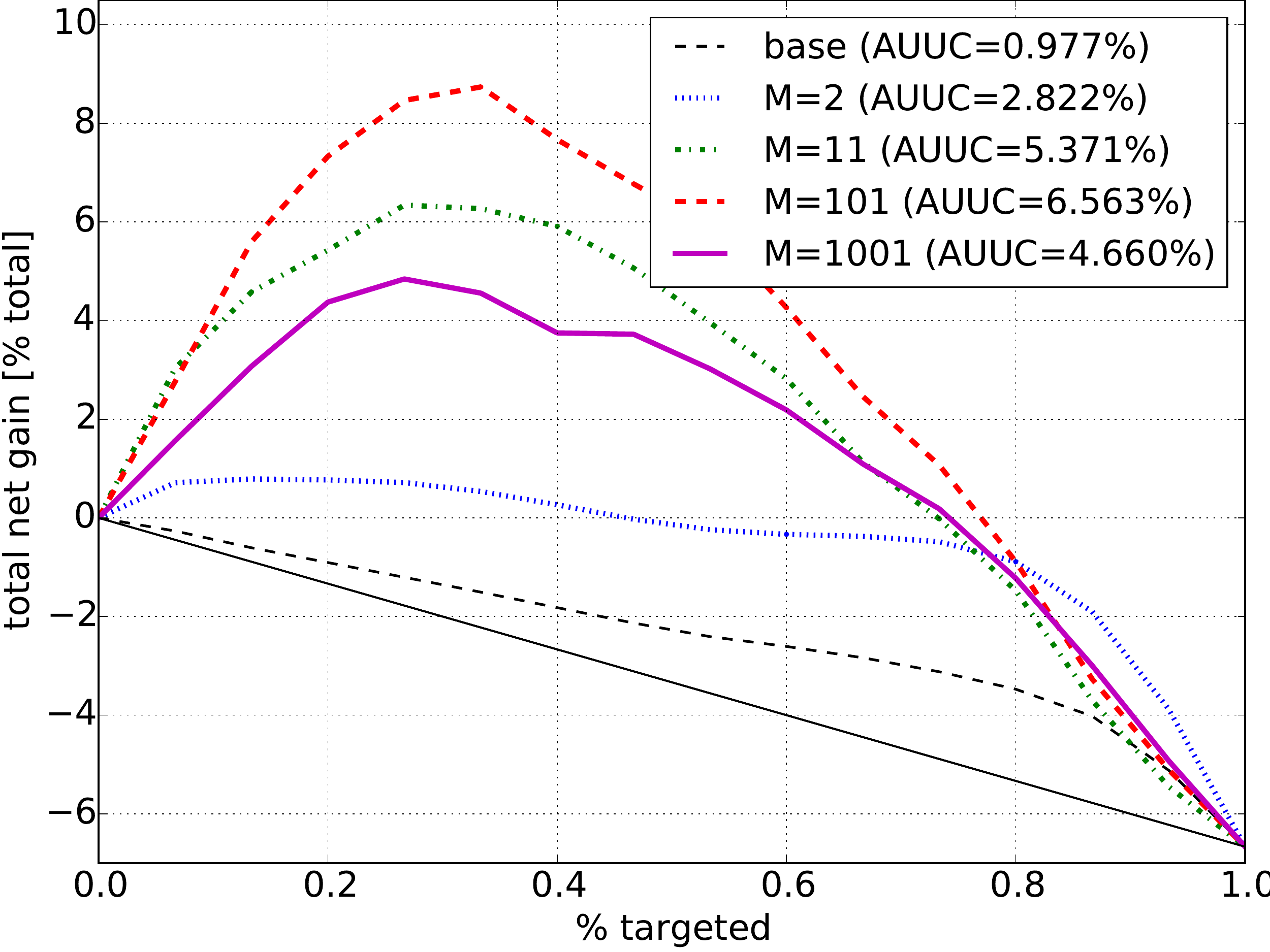}
    \end{tabular}
    }
  \end{center}
  \caption{Uplift curves for three uplift boosting algorithms with
    growing ensemble size for the {\tt veteran} dataset with
    E-divergence stump used as base model.}\label{fig:veteran}
\end{figure}
\begin{figure}[t]
\begin{center}
    {\footnotesize
    \begin{tabular}{ccc}
      ~~uplift AdaBoost & ~~balanced uplift boosting & ~~balanced forgetting boosting\\
  \includegraphics[width=0.3\textwidth]{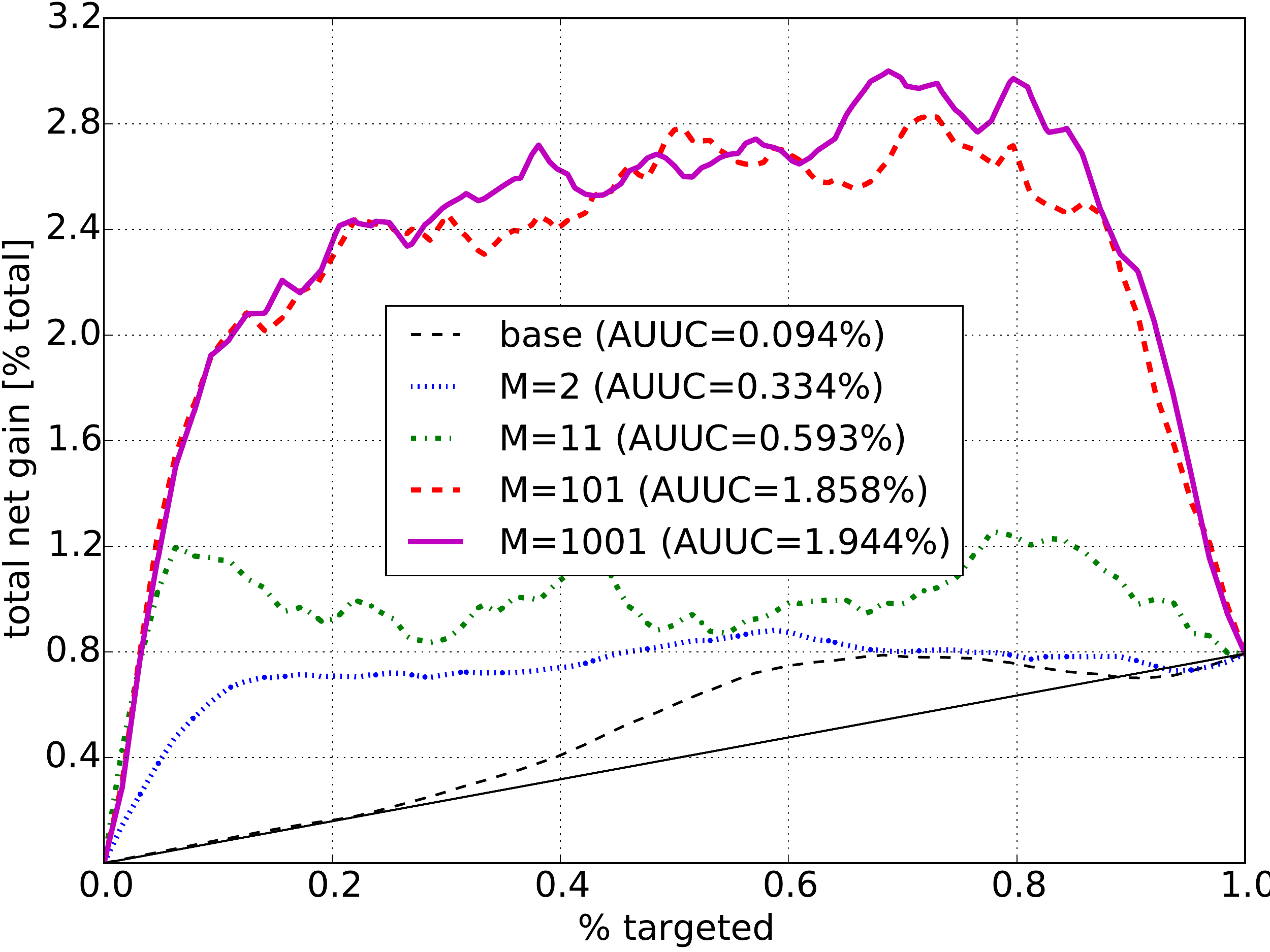}&
  \includegraphics[width=0.3\textwidth]{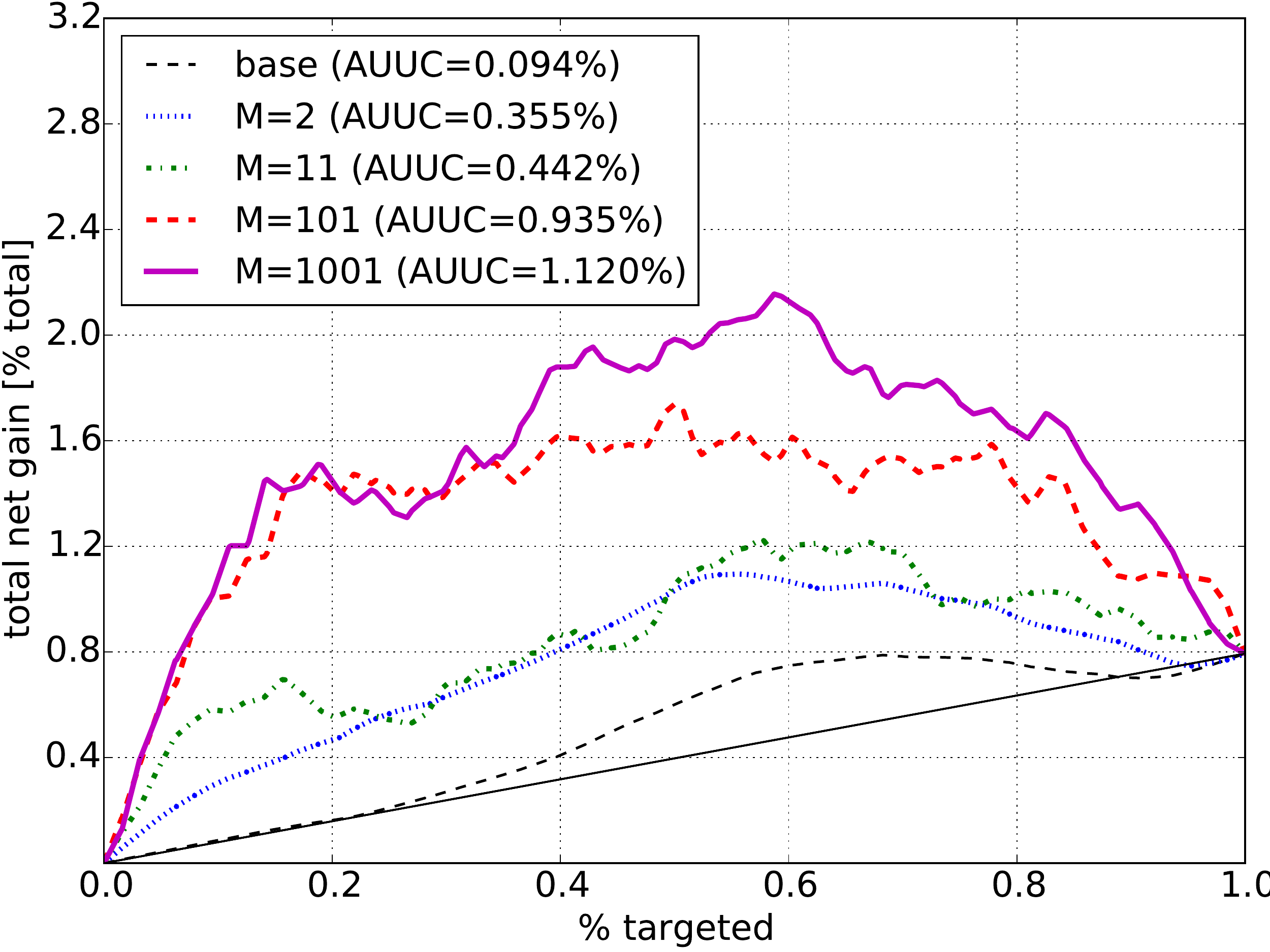}&
  \includegraphics[width=0.3\textwidth]{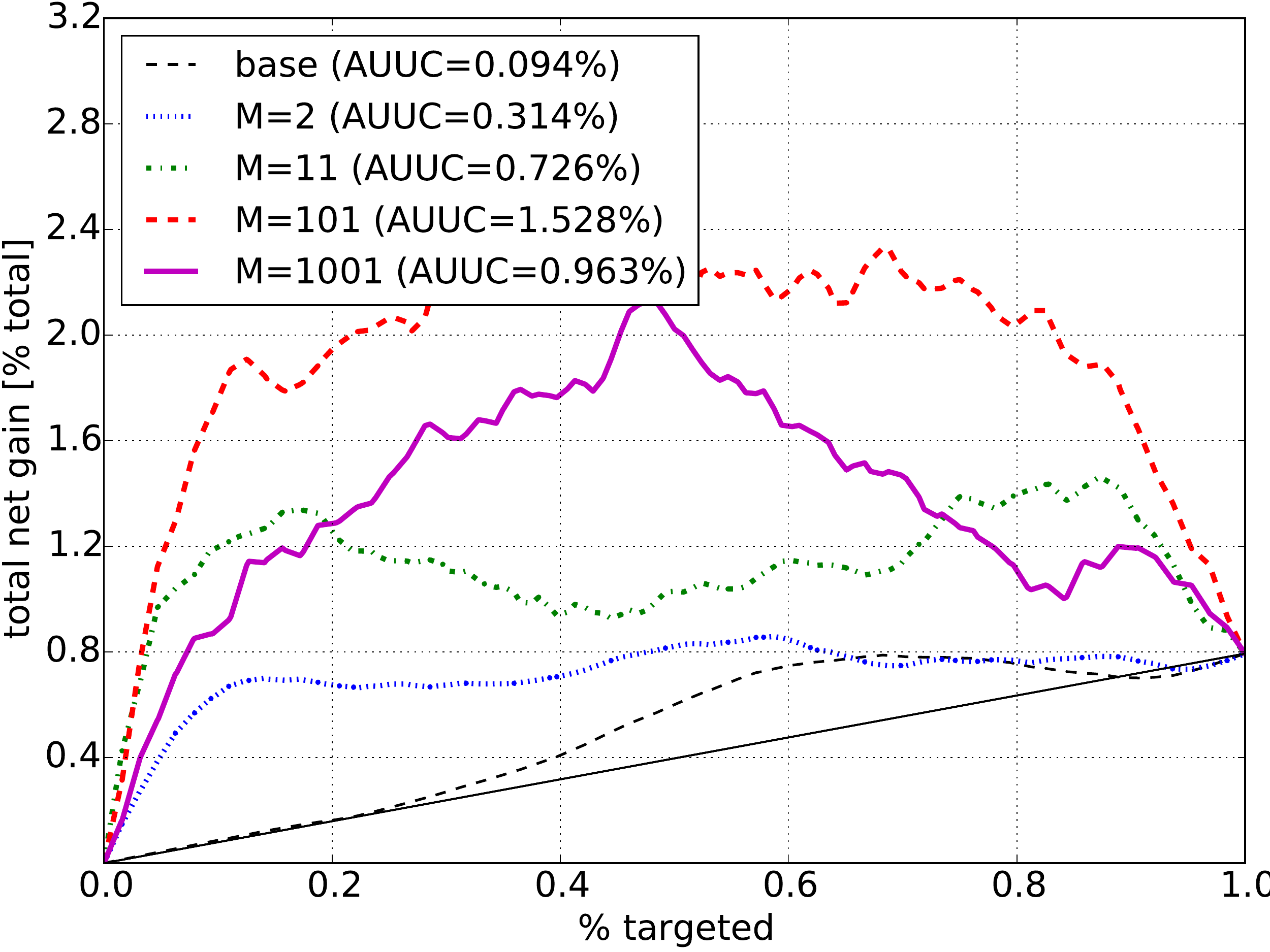}
    \end{tabular}
}
\end{center}
  \caption{Uplift curves for three uplift boosting algorithms with
    growing ensemble size for the {\tt colon death} dataset with three
    level E-divergence uplift tree used as base model.}
\label{fig:colon-death}
\end{figure}

Clearly, boosting dramatically improves performance in both cases.
For example, for the {\tt veteran} dataset, the treatment effect is,
overall, detrimental with 6\% lower survival rate; using just the base
model fails to improve the situation.  However, if stumps based uplift
AdaBoost is used one can select about 35\% of the population for whom
the treatment is highly beneficial.  The overall survival rate can be
improved by over 10\% (of the total population).  One can also see
that balanced uplift boosting also resulted in steady increase of
performance with the growing ensemble size, but the overall
improvement was smaller.  For balanced forgetting boosting the
performance initially improved, but at some point started to actually
decrease.  Very similar results can be seen in
Figure~\ref{fig:colon-death}, where boosting also resulted in a
dramatic performance increase.

Before discussing those results, we will show
Figures~\ref{fig:auucs-stump} and~\ref{fig:auucs-e8} depicting AUUCs
for growing ensemble sizes for the three proposed uplift boosting
algorithms on all datasets used in our experiments.

\begin{figure}[t]
  \begin{center}
    \includegraphics[width=\textwidth]{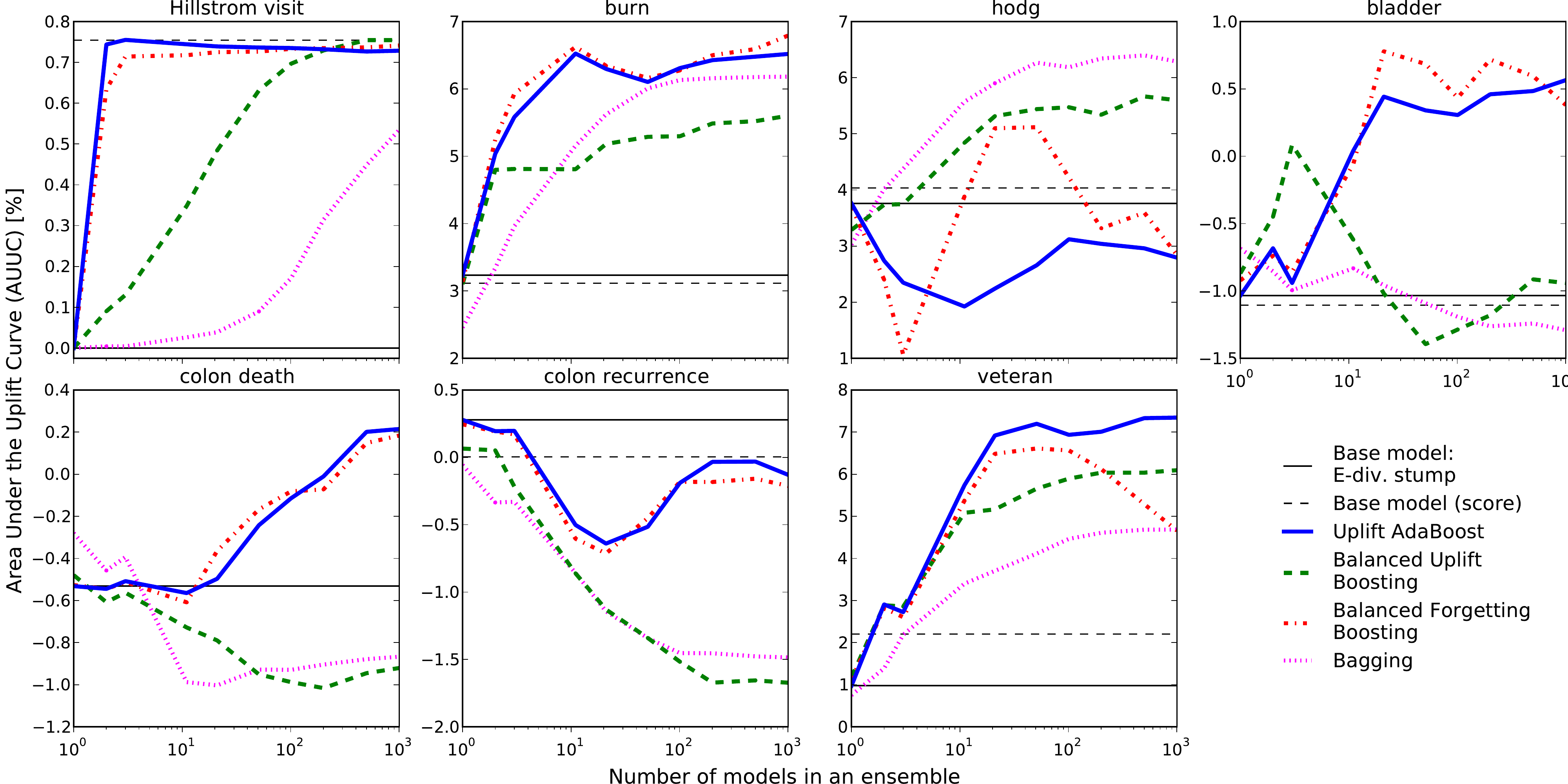}
  \end{center}
  \caption{AUUC versus ensemble size for the three proposed uplift
    boosting algorithms with E-divergence stumps used as base models.}
  \label{fig:auucs-stump}
\end{figure}

\begin{figure}[t]
  \begin{center}
    \includegraphics[width=\textwidth]{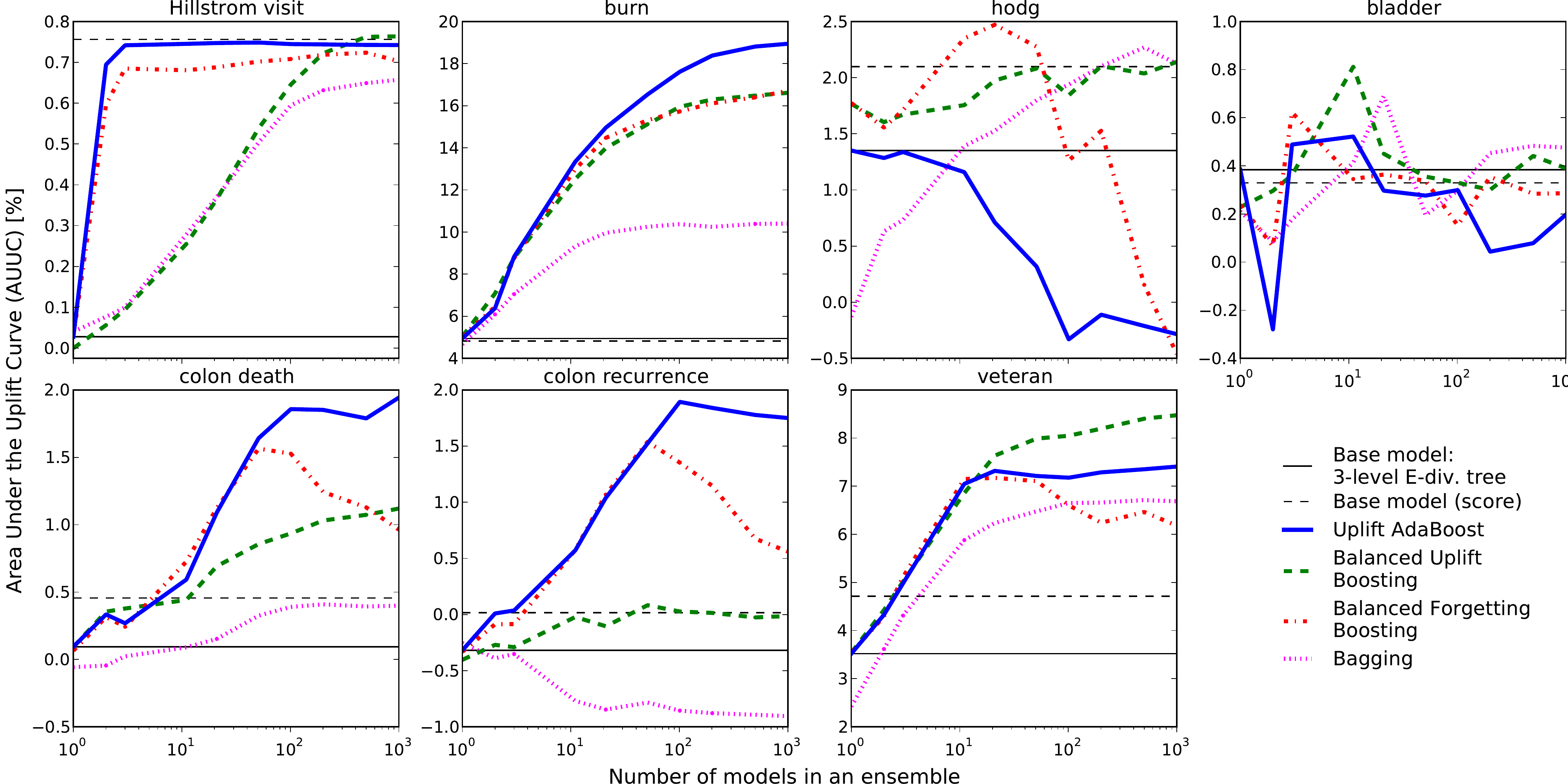}
  \end{center}
  \caption{AUUC versus ensemble size for the three proposed uplift
    boosting algorithms with three level E-divergence trees used as
    base models.}
  \label{fig:auucs-e8}
\end{figure}

The figures clearly demonstrate usefulness of boosting which usually
outperforms the base models and bagging, often dramatically. Two exceptions are the {\tt hodg} dataset and the {\tt bladder} dataset with
three level trees used as base learners for which bagging has an
advantage over all variants of boosting.  On those datasets as well as
on {\tt Hillstrom visit} and {\tt colon recurrence} with stumps
boosting did not improve over the base models.  On the {\tt Hillstrom
  visit} dataset boosting did dramatically improve on the base models'
0-1 predictions but not on its scores, which deserves a comment.  The
mailing campaign was overall effective so the base model's 0-1
predictions are always 1 giving a diagonal uplift curve and zero AUUC
(recall that we subtract the area under the diagonal).  There seems to
be only one predictive attribute in the data so stumps produced
excellent scores which were matched by boosted 0-1 stumps.  In this
sense boosting probably did find an optimal model in this case.

Interesting cases are the {\tt bladder} and {\tt colon death}
datasets with uplift stumps used as base models and {\tt colon recurrence} with
three level uplift trees, where uplift AdaBoost and balanced forgetting
boosting were able to achieve good performance even though the base
model has negative AUUC. This was not the case for bagging and
balanced uplift boosting.

Let us now compare the three proposed uplift boosting algorithms.  No
single algorithm produces the best results in all cases.  It can
however be seen that typically uplift AdaBoost and balanced
forgetting boosting outperform balanced uplift boosting.  One can
conclude that the property of forgetting the last member added to the
ensemble is important, probably because it leads to higher ensemble
diversity.

Unfortunately, the balanced forgetting boosting, which performs well
for smaller ensembles often begins deteriorating when the ensemble
grows too large.  We conjecture that the reason is that the bound on
the uplift analogue of classification error
(Equation~\ref{eq:ub-err-bound}) is not guaranteed to decrease (see
Section~\ref{sec:FBW}).  The performance of the two other uplift
boosting algorithms is much more stable in this respect.

Balance seems to be a less important condition.  However there is also
a possibility that balance is indeed important but uplift AdaBoost
does, in practice, maintain approximate balance between treatment and
control groups.  To investigate this issue we computed the value
$$\max_{m=1,\ldots,101}\max\left\{\frac{p^T_m}{p^C_m},\frac{p^C_m}{p^T_m}\right\}$$
for uplift AdaBoost for each dataset and base model.  The value gives
us the largest relative difference between the weight of treatment and
control groups over the first 101 iterations.  The largest values were
obtained for the {\tt hodg} dataset with stumps ($6.56$) and {\tt
  burn} ($24.50$), {\tt hodg} ($10.93$), {\tt veteran} ($7.28$) with
three level uplift trees; for all other datasets and base models the
value was below $2.3$.  One can see in Figures~\ref{fig:auucs-stump}
and~\ref{fig:auucs-e8} that (except for {\tt burn}) those are the
datasets for which balanced uplift boosting outperformed uplift
AdaBoost.  Balanced forgetting boosting also performed well on those
datasets, at least until its performance started to degrade with
larger ensemble sizes.

We thus conclude that balance is in fact important, however in many
practical cases the boosting process does not violate it too much and
countermeasures are not necessary.  There are, however, situations where
this is not the case and uplift boosting which explicitly maintains
balance is preferred.

\section{Conclusions}

In this paper we have developed three boosting algorithms for the
uplift modeling problem.  We began by formulating three properties
which uplift boosting algorithms should have. Since all three
cannot be simultaneously satisfied, we designed three algorithms, each
satisfying two of the properties.  The properties of the proposed
algorithms are briefly summarized in Table~\ref{tab:alg-prop}.

\begin{table}
  \caption{Properties of the proposed uplift boosting algorithms.}\label{tab:alg-prop}
  \begin{center}
    \begin{tabular}{|l|c|c|c|}
      \hline
          & nonincreasing &         & forgetting  \\
Algorithm & training error bound        & balance & the last member \\\hline
uplift AdaBoost & Yes & No & Yes\\
balanced uplift boosting & Yes & Yes & No\\
balanced forgetting boosting & No & Yes & Yes\\
\hline
    \end{tabular}
  \end{center}
\end{table}

Experimental evaluation showed that boosting has the potential to
dramatically improve the performance of uplift models and typically
performs significantly better than bagging.  Of the three algorithms
uplift AdaBoost usually (but not always) performed best.  This is most
probably due to the forgetting and training error reduction properties
it satisfies.  A more thorough analysis revealed that the balance
condition, which uplift AdaBoost does not satisfy explicitly, is often
satisfied approximately; when this is not the case, other uplift
boosting algorithms, especially balanced uplift boosting, are more
appropriate.  The balanced forgetting boosting which satisfies the
balance property and also forgets the last ensemble member does not
guarantee the decrease of training error bound.  As a result, the
algorithm performs well for small ensembles but diverges as more
members are added.

\appendix
\section*{Appendix A.}
\label{app:thm:p_equal_q}

\begin{proof}(of Theorem~\ref{thm:forgetting-interp})
Note that the assumption of random group assignment implies
$P^T(h=1)=P^C(h=1)=P(h=1)$ since both groups are scored with the same
model and have the same distributions of predictor variables.  Using
the balance condition, the error $\epsilon$ of $h$, defined
in Equation~\ref{eq:err-upl}, can be expressed as (the second equality follows from $p^T=p^C=\frac{1}{2}$)
\begin{align*}
2\epsilon & = 2P^T(h=1-y)p^T + 2P^C(h=y)p^C = P^T(h=1-y) + P^C(h=y) \\
& = P^T(h=1, y=0) + P^T(h=0, y=1) + P^C(h=y=0) + P^C(h=y=1)\\
& = P^T(y=0|h=1)P^T(h=1) + P^T(y=1|h=0)P^T(h=0)\\ &\qquad\qquad + P^C(y=1|h=1)P^C(h=1) + P^C(y=0|h=0)P^C(h=0). \\
\intertext{Using the assumption of random treatment assignment and rearranging:}
& = P(h=1)\left[P^T(y=0|h=1) + P^C(y=1|h=1)\right]\\ &\qquad\qquad + P(h=0)\left[P^T(y=1|h=0) + P^C(y=0|h=0)\right] \\
& = P(h=1)\left[1-P^T(y=1|h=1) + P^C(y=1|h=1)\right]\\ &\qquad\qquad + P(h=0)\left[P^T(y=1|h=0) + 1-P^C(y=1|h=0)\right] \\
& = 1 + P(h=1) \left[-(P^T(y=1|h=1) - P^C(y=1|h=1))\right]\\ &\qquad\qquad + P(h=0)\left[P^T(y=1|h=0)-P^C(y=1|h=0)\right].
\end{align*}
After taking $\epsilon = \frac{1}{2}$ the result follows.
\end{proof}

\appendix
\section*{Appendix B.}
\label{app:thm:bub_solution}

Here we provide a proof of Theorem~\ref{thm:bub_solution}. First we
introduce a lemma establishing relations between $\beta_m^C$ and
$\beta_m^T$:

\begin{lemma}\label{lem:relations_bmt_bmc}
The following equivalence holds:
\begin{equation}\label{eq:beta-gt-beta}
\beta_m^C \geq \beta_m^T \Leftrightarrow \left( \epsilon_m^C \geq \epsilon_m^T \land \beta_m^C \in [1,+\infty) \right) \lor \left( \epsilon_m^C \leq \epsilon_m^T \land \beta_m^C \in (0,1] \right).
\end{equation}
\end{lemma}
\begin{proof}
From Equation~\ref{eq:bub_notation} we conclude that $\beta_m^C \geq
\beta_m^T$ is equivalent to $(1 - a_m) \beta_m^C \geq b_m$.  Also note
that $b_m/(1-a_m)=1$. 
Clearly
\[
\beta_m^C \geq \beta_m^T\Leftrightarrow
 \underbrace{(\beta_m^C \geq \beta_m^T\wedge \epsilon_m^C > \epsilon_m^T)}_{(a)}
 \vee \underbrace{(\beta_m^C \geq \beta_m^T\wedge \epsilon_m^C < \epsilon_m^T)}_{(b)}
  \vee \underbrace{(\beta_m^C \geq \beta_m^T\wedge \epsilon_m^C = \epsilon_m^T)}_{(c)}.
\]
Condition (a) is equivalent to $(1 - a_m) \beta_m^C \geq
b_m\wedge\epsilon_m^C > \epsilon_m^T\wedge 1 - a_m > 0$ which in turn
is equivalent to $\beta_m^C \geq 1\wedge\epsilon_m^C > \epsilon_m^T$.
Similarly (b) is equivalent to $\beta_m^C \leq 1\wedge\epsilon_m^C <
\epsilon_m^T$.  For (c), notice that $e_m^C = e_m^T \Leftrightarrow
\beta_m^C = \beta_m^T$ (Equation~\ref{eq:bub_notation}) so the right
hand side of Equation~\ref{eq:beta-gt-beta} becomes $\beta_m^C \in
(0,+\infty)\wedge e_m^C = e_m^T$, which is trivially true by nonnegativity of $\beta_m^C$.
The result follows after taking the conjunction
\[
(\beta_m^C \geq 1\wedge\epsilon_m^C > \epsilon_m^T)\vee(\beta_m^C \leq
1\wedge\epsilon_m^C < \epsilon_m^T)\vee(\beta_m^C \in
(0,+\infty)\wedge e_m^C = e_m^T).
\]
\end{proof}
It is easy to see that the complementary condition is:
\begin{equation}
\beta_m^C < \beta_m^T \Leftrightarrow \left( \epsilon_m^C < \epsilon_m^T \land \beta_m^C \in (1,+\infty) \right) \lor \left( \epsilon_m^C > \epsilon_m^T \land \beta_m^C \in (0,1) \right).
\end{equation}

\begin{proof}(of Theorem~\ref{thm:bub_solution})
Consider two cases: $\beta_m^C < \beta_m^T$ and $\beta_m^C \geq
\beta_m^T$.  Assume first, that $\beta_m^C < \beta_m^T$.  We have
$$f\left( \beta_m^C \right) = f^{C<T}\left( \beta_m^C \right) = \frac{ 1 - (1 - \epsilon_m^C)(1 - \beta_m^C)}{\sqrt{\beta_m^C}}.$$
Obviously $f^{C<T}\left( 1 \right) = 1$. Now we take the derivative:
$$\frac{df^{C<T}}{d\beta_m^C} = \frac{(1 -
  \epsilon_m^C)\sqrt{\beta_m^C} -
  \frac{1}{2\sqrt{\beta_m^C}}(\epsilon_m^C + \beta_m^C - \epsilon_m^C
  \beta_m^C)}{\beta_m^C},$$ and by equating to zero find a local
minimum at
\[\beta_m^C = \frac{\epsilon_m^C}{1 - \epsilon_m^C},\qquad \beta_m^T =  \frac{2\epsilon_m^C - \epsilon_m^T}{1 - \epsilon_m^T}.
\]
Since the numerator of the derivative is a linear function of
$\beta_m^C$, this is also a global minimum of $f^{C<T}$.

Consider now the second case: $\beta_m^C \geq \beta_m^T$.  Analogously we have
$$f^{C\geq T}\left( \beta_m^C \right) = \frac{ 1 - (1 - \epsilon_m^C)(1 - \beta_m^C)}{\sqrt{a_m\beta_m^C + b_m}},$$
with $f^{C\geq T}\left( 1 \right) = 1$ and a global minimum at
\[
\beta_m^C = \frac{2\epsilon_m^T - \epsilon_m^C}{1 - \epsilon_m^C},
\qquad \beta_m^T = \frac{\epsilon_m^T}{1 - \epsilon_m^T},
\]
provided that $\sqrt{a_m\beta_m^C + b_m} > 0$. This, however, is always true for the optimal $\beta_m^C$.
%
Hence, the objective function can be defined as:
$$f(\beta_m^C)=\left\{
\begin{array}{ll}
f^{C\geq T}(\beta_m^C)& if \; \beta_m^C \geq \beta_m^T, \\
f^{C < T}(\beta_m^C)& if \; \beta_m^C < \beta_m^T, \\
1& if \; \beta_m^C = 1.
\end{array}
\right.$$
We assume strict inequalities between errors and
$\frac{1}{2}$, otherwise the optimum is at $\beta_m^C=\beta_m^T=1$
with $f(1)=1$.
To minimize this function we are going to consider several
cases:




\begin{enumerate}
\item $\epsilon_m^C \leq \epsilon_m^T < \frac{1}{2}$.
If $\beta_m^C \in (0,1]$ we have $\beta_m^C \geq \beta_m^T$ (from
  Lemma~\ref{lem:relations_bmt_bmc}) so $f(\beta_m^C) = f^{C\geq
    T}(\beta_m^C)$ and the optimum is at $\beta_m^{C*} =
  \frac{2\epsilon_m^T-\epsilon_m^C}{1-\epsilon_m^C}$. Since
  $\epsilon_m^C < 2\epsilon_m^T < 1$ we have $0 < \beta_m^{C*} < 1$
  and there is a local minimum of $f$ in $\beta_m^{C*}$.

If $\beta_m^C \in (1,+\infty)$ we have $\beta_m^C < \beta_m^T$ (from
Lemma~\ref{lem:relations_bmt_bmc}) so $f(\beta_m^C) =
f^{C<T}(\beta_m^C)$. Since the optimum of $f^{C<T}$ is in
$\frac{\epsilon_m^C}{1-\epsilon_m^C} < 1$, $f$ is increasing on
$(1,+\infty)$ with minimum at $f(1)=1$.

Thus we have only one minimum of the objective function $f$ in the
whole domain, which is in $\beta_m^{C*}$. The corresponding upper
bound is $f(\beta_m^{C*}) = 2\sqrt{\epsilon_m^T(1-\epsilon_m^T)}$.

\item $\frac{1}{2} < \epsilon_m^C < \epsilon_m^T$. All derivations are
  analogous to the previous case, but now the functions: $f^{C\geq
    T}$, $f^{C<T}$ have their minima taken over $(1,+\infty)$. Hence,
  only $\beta_m^{C*} = \frac{\epsilon_m^C}{1-\epsilon_m^C} > 1$ is a
  valid minimum of $f$ on $(0,+\infty)$. The upper bound is now:
  $f(\beta_m^{C*}) = 2\sqrt{\epsilon_m^C(1-\epsilon_m^C)}$.

\item $\epsilon_m^T < \epsilon_m^C < \frac{1}{2}$.  The proof is
  similar to cases 1 and 2, minimum equal to
  $2\sqrt{\epsilon_m^C(1-\epsilon_m^C)}$.


\item $\frac{1}{2} < \epsilon_m^T < \epsilon_m^C$. The proof is similar to cases 1 and 2, minimum equal to
  $2\sqrt{\epsilon_m^T(1-\epsilon_m^T)}$.


\item $\epsilon_m^C < \frac{1}{2} < \epsilon_m^T$.  Now $f(\beta_m^C)
  = f^{C\geq T}(\beta_m^C)$ for $\beta_m^C \in (0,1]$ and
    $f(\beta_m^C) = f^{C<T}(\beta_m^C)$ for $\beta_m^C \in
    (1,+\infty)$. Yet, the optimum for the first case is in
    $\frac{2\epsilon_m^T-\epsilon_m^C}{1-\epsilon_m^C} > 1$ and for
    the second in $\frac{\epsilon_m^C}{1-\epsilon_m^C} < 1$. Since
    both optima are outside the ranges implied by
    Lemma~\ref{lem:relations_bmt_bmc}, the minimum is at $f(1)=1$.

\item $\epsilon_m^T < \frac{1}{2} < \epsilon_m^C$.  Proof analogous to case 5.


\end{enumerate}
\end{proof}

\bibliographystyle{plain}      
\bibliography{boosting}

\begin{thebibliography}{10}

\bibitem{csiszar}
I.~Csiszar and P.~Shields.
\newblock Information theory and statistics: A tutorial.
\newblock {\em Foundations and Trends in Communications and Information
  Theory}, 1(4):417--–528, 2004.

\bibitem{freund:97}
Y.~Freund and R.E. Schapire.
\newblock A decision-theoretic generalization of on-line learning and an
  application to boosting.
\newblock {\em Journal of Computer and System Sciences}, 55(1):119--139, 1997.

\bibitem{GuelmanRandForest}
L.~Guelman, M.~Guill\'{e}n, and A.M. P\'{e}rez-Mar\'{i}n.
\newblock Random forests for uplift modeling: An insurance customer retention
  case.
\newblock In {\em Modeling and Simulation in Engineering, Economics and
  Management}, volume 115 of {\em Lecture Notes in Business Information
  Processing (LNBIP)}, pages 123--133. Springer, 2012.

\bibitem{Guelman2}
L.~Guelman, M.~Guill\'{e}n, and A.M. P\'{e}rez-Mar\'{i}n.
\newblock A survey of personalized treatment models for pricing strategies in
  insurance.
\newblock {\em Insurance: Mathematics and Economics}, 58:68--76, 2014.

\bibitem{Hansotia2002}
B.~Hansotia and B.~Rukstales.
\newblock Incremental value modeling.
\newblock {\em Journal of Interactive Marketing}, 16(3):35--46, 2002.

\bibitem{Hillstrom2008}
K.~Hillstrom.
\newblock The {MineThatData} e-mail analytics and data mining challenge.
\newblock MineThatData blog,
  \url{http://blog.minethatdata.com/2008/03/minethatdata-e-mail-analytics-and-data.html},
  2008.
\newblock Retrieved on 06.10.2014.

\bibitem{Holland86}
P.W. Holland.
\newblock Statistics and causal inference.
\newblock {\em Journal of the American Statistical Association},
  81(396):945--960, December 1986.

\bibitem{my:upl-surv}
S.~Jaroszewicz and P.~Rzepakowski.
\newblock Uplift modeling with survival data.
\newblock In {\em ACM SIGKDD Workshop on Health Informatics (HI-KDD'14)}, New
  York, August 2014.

\bibitem{my:uplift-clinical-ml}
M.~Ja{\'s}kowski and S.~Jaroszewicz.
\newblock Uplift modeling for clinical trial data.
\newblock In {\em ICML 2012 Workshop on Machine Learning for Clinical Data
  Analysis}, Edinburgh, Scotland, June 2012.

\bibitem{Lo2002}
V.~Lo.
\newblock The true lift model - a novel data mining approach to response
  modeling in database marketing.
\newblock {\em SIGKDD Explorations}, 4(2):78--86, 2002.

\bibitem{pechyony13}
D.~Pechyony, R.~Jones, and X.~Li.
\newblock A joint optimization of incrementality and revenue to satisfy both
  advertiser and publisher.
\newblock In {\em WWW 2013 Companion Publication}, pages 123--124, 2013.

\bibitem{RadcliffeTechRep}
N.J. Radcliffe and P.D. Surry.
\newblock Real-world uplift modelling with significance-based uplift trees.
\newblock Portrait Technical Report TR-2011-1, Stochastic Solutions, 2011.

\bibitem{Robins2004}
J.~Robins and A.~Rotnitzky.
\newblock Estimation of treatment effects in randomised trials with
  non-compliance and a dichotomous outcome using structural mean models.
\newblock {\em Biometrika}, 91(4):763--783, 2004.

\bibitem{my:uplift-trees}
P.~Rzepakowski and S.~Jaroszewicz.
\newblock Decision trees for uplift modeling.
\newblock In {\em Proc. IEEE International Conference on Data Mining (ICDM)},
  pages 441--450, Sydney, December 2010.

\bibitem{my:uplift-trees-KAIS}
P.~Rzepakowski and S.~Jaroszewicz.
\newblock Decision trees for uplift modeling with single and multiple
  treatments.
\newblock {\em Knowledge and Information Systems}, 32:303--327, August 2012.

\bibitem{Schapire:1990}
R.~Schapire.
\newblock The strength of weak learnability.
\newblock {\em Machine Learning}, 5(2):197--227, July 1990.

\bibitem{schapire1999improved}
R.~Schapire and Y.~Singer.
\newblock Improved boosting algorithms using confidence-rated predictions.
\newblock {\em Machine learning}, 37(3):297--336, 1999.

\bibitem{my:uplEnsembles}
M.~So{\l}tys, S.~Jaroszewicz, and P.~Rzepakowski.
\newblock Ensemble methods for uplift modeling.
\newblock {\em Data Mining and Knowledge Discovery}, 29(6):1531--1559, 2015.

\bibitem{Goetghebeur2003}
S.~Vansteelandt and E.~Goetghebeur.
\newblock Causal inference with generalized structural mean models.
\newblock {\em Journal of the Royal Statistical Society~B}, 65(4):817--–835,
  2003.

\bibitem{Webb2000}
G.~I. Webb.
\newblock Multiboosting: A technique for combining boosting and wagging.
\newblock {\em Machine Learning}, 40:159--196, 2000.

\end{thebibliography}

\end{document}